\pdfoutput=1 
\documentclass[12pt,cleveref]{colt2023} % Anonymized submission
\PassOptionsToPackage{capitalize}{cleveref}

\usepackage{microtype}
\usepackage{graphicx}
\usepackage{wrapfig}
\usepackage{subcaption}
\usepackage{booktabs} % for professional tables

\usepackage{hyperref}

\usepackage{amsfonts,thmtools}
\usepackage{mathtools}

\usepackage{natbib} % Added this for the arxiv version

\usepackage{cleveref}

\newcommand{\proplabel}[1]{\label[proposition]{#1}}
\newcommand{\lemmalabel}[1]{\label[lemma]{#1}}
\newcommand{\corolabel}[1]{\label[corollary]{#1}}

\usepackage[nolist,nohyperlinks]{acronym}
\makeatletter
\AtBeginDocument
 {
   \def\ltx@label#1{\cref@label{#1}}%add braces
   \def\label@in@display@noarg#1{\cref@old@label@in@display{#1}}%remove braces  
 } %
\makeatother

\usepackage{bbm}
\usepackage{xcolor}                     % colors
\usepackage{adjustbox}

\usepackage{tikz}

\usetikzlibrary{
    calc,           % For coordinate calculations
    arrows.meta,    % For nicer arrowheads (e.g., Stealth)
    positioning     % For placing labels
}

\newtheorem{assumption}[theorem]{Assumption}

\DeclareMathOperator*{\argmin}{arg\,min}
\newcommand{\inner}[2]{\left\langle#1,\,#2\right\rangle} %nabla fx
\newcommand{\KL}{\mathrm{KL}} %nabla fx
\DeclarePairedDelimiter\floor{\lfloor}{\rfloor}

\newcommand{\bigO}{\mathcal{O}} 

\newcommand{\Var}{\mathrm{Var}} 
\newcommand{\Cov}{\mathrm{Cov}} 
\newcommand{\bE}{\mathbb{E}} 
\newcommand{\E}{\mathbb{E}}
\newcommand{\tr}{\operatorname{tr}}
\newcommand{\cZ}{\mathcal{Z}}

\newcommand{\TV}{\mathrm{TV}}

\newcommand{\bR}{\mathbb{R}}
\newcommand{\bN}{\mathbb{N}} 
\newcommand{\cX}{\mathcal{X}} 
 
\newcommand{\rP}{\mathrm{P}}
\newcommand{\rQ}{\mathrm{Q}}
\newcommand{\cP}{\mathcal{P}}

\newcommand{\cA}{\mathcal{A}}

\newcommand{\diag}{\mathrm{diag}}

\newcommand{\Si}{S^{(i)}} % Perturbed dataset
\newcommand{\xtS}{x_t(S)}
\newcommand{\xtSi}{x_t(\Si)}

\newcommand{\stabP}{\varepsilon_{\mathrm{pstab}}} % parameter stability

\newcommand{\Proj}{\mathrm{Proj}}

\newcommand{\cond}{\kappa}

\newcommand{\cstSmooth}{\beta}

\newcommand{\cstPL}{\mu}
\newcommand{\cstScvx}{\alpha}
\newcommand{\cstCon}{r}
\newcommand{\cstSA}{C_{\ell, P}}

\newcommand{\lmin}{\lambda_{\mathrm{min}}}
\newcommand{\lmax}{\lambda_{\mathrm{max}}}

\newcommand{\cN}{\mathcal{N}} 
\newcommand{\cM}{\mathcal{M}} 
\newcommand{\pr}[1]{\left( #1 \right)}
\newcommand{\br}[1]{\left[ #1 \right]}

\newcommand{\abs}[1]{\left|#1\right|}
\newcommand{\tp}{^{\top}}
\newcommand{\ip}[1]{\left\langle #1 \right\rangle}
\begin{acronym}
  \acro{PAC}{Probably Approximately Correct}
\acro{RLS}{Regularized Least Squares}
\acro{ERM}{Empirical Risk Minimization}
\acro{RKHS}{Reproducing kernel Hilbert space}
\acro{DA}{Domain Adaptation}
\acro{PD}{Positive Definite}
\acro{PSD}{Positive Semi-Definite}
\acro{SGD}{Stochastic Gradient Descent}
\acro{PSGD}{Preconditioned Stochastic Gradient Descent}
\acro{OGD}{Online Gradient Descent}
\acro{GD}{Gradient Descent}
\acro{SGLD}{Stochastic Gradient Langevin Dynamics}
\acro{IW}{Importance Weighted}
\acro{MGF}{Moment-Generating Function}
\acro{ES}{Efron-Stein}
\acro{ESS}{Effective Sample Size}
\acro{KL}{Kullback-Leibler}
\acro{SVD}{Singular Value Decomposition}
\acro{PL}{Polyak-Łojasiewicz}
\acro{NTK}{Neural Tangent Kernel}
\acro{KLS}{Kernelized Least-Squares}
\acro{KRLS}{Kernelized Regularized Least-Squares}
\acro{ReLU}{Rectified Linear Unit}
\acro{RF}{Random Feature}
\acro{TIC}{Takeuchi Information Criterion}
\acro{MLE}{Maximum Likelihood Estimation}
\end{acronym}

\title[On-Average Stability of Multipass Preconditioned SGD and Effective Dimension]{On-Average Stability of Multipass Preconditioned \\SGD and Effective Dimension}

\usepackage{times}

\coltauthor{%
 \Name{Simon Vary} \Email{simon.vary@stats.ox.ac.uk}\\
 \addr Department of Statistics, University of Oxford
 \AND
 \Name{Tyler Farghly} \Email{farghly@stats.ox.ac.uk}\\
 \addr Department of Statistics, University of Oxford%
 \AND
 \Name{Ilja Kuzborskij} \Email{iljak@google.com}\\
 \addr Google DeepMind%
 \AND
 \Name{Patrick Rebeschini} \Email{patrick.rebeschini@stats.ox.ac.uk}\\
 \addr Department of Statistics, University of Oxford%
}

\begin{document}

\maketitle

\vspace{0em}
\begin{abstract}
  We study trade-offs between the population risk curvature, geometry of the noise, and preconditioning on the generalisation ability of the multipass Preconditioned Stochastic Gradient Descent (PSGD). Many practical optimisation heuristics implicitly navigate this trade-off in different ways --- for instance, some aim to whiten gradient noise, while others aim to align updates with expected loss curvature. When the geometry of the population risk curvature and the geometry of the gradient noise do not match, an aggressive choice that improves one aspect can amplify instability along the other, leading to suboptimal statistical behavior. In this paper we employ \emph{on-average algorithmic stability} to connect generalisation of PSGD to the \emph{effective dimension} that depends on these sources of curvature. While existing techniques for on-average stability of SGD are limited to a single pass, as first contribution we develop a new on-average stability analysis for multipass SGD that handles the correlations induced by data reuse. This allows us to derive excess risk bounds that depend on the effective dimension. In particular, we show that an improperly chosen preconditioner can yield suboptimal effective dimension dependence in both optimisation and generalisation. Finally, we complement our upper bounds with matching, instance-dependent lower bounds.
\end{abstract}


\section{INTRODUCTION}
Training of machine learning models is usually posed as a minimisation of the population risk. In particular, given a data distribution $\rQ$ supported on the example space $\mathcal{Z}$, the goal is to minimise the \emph{population risk} $f$. The population risk and its empirical counterpart are defined as,
\begin{equation*}
  f(x) = \bE_{z\sim \rQ}[\ell(x, z)]~, \qquad f_S(x) = \frac1n \sum_{i=1}^n \ell(x, z_i)~,
\end{equation*}
respectively, where $\ell$ is a smooth loss function parameterized by $x$ and evaluated on an example $z$.
In the standard setting, where the data distribution is unknown and we have access only to a finite training set $S=\left\{z_1, \ldots, z_n\right\}\subset\mathcal{Z}$ of $n$-samples drawn i.i.d.\ from $\rQ$, we instead minimise the \emph{empirical risk} $f_S$.
Given the solution $\hat x$ returned by an algorithm, its generalization ability is captured by the \emph{excess risk},
\begin{align*}
  \E[\delta f(\hat x)] \qquad \text{where} \qquad \delta f(x) = f(x) - \inf\nolimits_{x \in \cX} f(x)~.
\end{align*}
In this work, we focus on preconditioned SGD, meaning that empirical risk is minimised iteratively by observing gradients on individual examples drawn uniformly from the training set:
\begin{equation}\label{eq:psgd}
  x_{t+1} = x_t - \eta_t \, P \, \nabla\ell(x_t, z_{i_t})~, \quad i_t \sim \text{Unif}(\{1, \ldots, n\})~, \quad t = 0,1,2,\ldots
\end{equation}
where $\eta_t$ is the step size, $P$ is a \ac{PD} \emph{preconditioning} matrix and $z_{i_t}\in S$ are sampled randomly from S uniformly with replacement.
Note that the update in \cref{eq:psgd} is often not limited to a single pass over the training set. Hence, in the present work, we consider \ac{PSGD} in the \emph{multipass} regime.
Since gradients are random variables, this randomised procedure is inevitably affected by the noise in stochastic gradients.
In this work we pay close attention to the geometry of gradient covariance considering the \emph{gradient covariance matrix} $\Sigma \succeq \Var_z(\nabla \ell(x,z))$.\footnote{Derivatives are always taken with respect to the first argument, unless stated otherwise.}

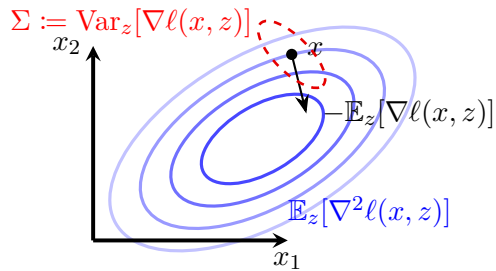
\begin{wrapfigure}{r}{0.44\textwidth}
  % Start the TikZ picture
% We define styles to make the code readable and easy to change
\begin{tikzpicture}[
    scale=0.45, % Scale down the entire figure to half size
    % Increased line thickness for all level sets
    levelset/.style={line width=1.2pt, opacity=0.8},
    newset/.style={red!90!black, line width=1pt, dashed}, % Style for the red ellipsoid
    gradient/.style={-Stealth[{length=2.5mm}], black, thick}, % Gradient arrow is black
    point/.style={fill=black, circle, inner sep=1.5pt},
    axis/.style={->, >=stealth, black, line width=1.5pt} % Increased line width for axes
]

    % --- Define rotation and point parameters ---
    \def\a{4} % major radius
    \def\b{2} % minor radius
    \def\angle{60} % angle for point x on ellipse
    \def\rot_angle{30} % rotation angle for main ellipses

    % --- 1. Main Ellipsoid Level Sets (4 lines) ---
    % We draw 4 concentric ellipses, rotated by \rot_angle
    % With "more blue" (darker) closer to the center
    \foreach \s/\color in {2/blue!90, 3/blue!70, 4/blue!50, 5/blue!30} {
        \draw [levelset, draw=\color, rotate=\rot_angle] (0,0) ellipse (\s cm and \s*0.5 cm);
    }
    
    % Add a node for the center (un-rotated) - REMOVED x_0 label
  %  \node [point] at (0,0) {};

    % --- 2. The Point x ---
    % We place the point 'x' on the third level set (s=4)
    % We must manually rotate the coordinate of the point
    % Original (x,y) = (a*cos(t), b*sin(t))
    % Rotated (x',y') = (x*cos(rot) - y*sin(rot), x*sin(rot) + y*cos(rot))
    \coordinate (x_point) at (
        {\a*cos(\angle)*cos(\rot_angle) - \b*sin(\angle)*sin(\rot_angle)}, 
        {\a*cos(\angle)*sin(\rot_angle) + \b*sin(\angle)*cos(\rot_angle)}
    );
    
    % Draw the point and label it
    \node (x_node) [point, label={[right=2pt]:$x$}] at (x_point) {};

    % --- 3. The Gradient Arrow ---
    % Calculate the un-rotated gradient vector components, scaled for visibility
    \pgfmathsetmacro{\gradxUnrotated}{cos(\angle)/\a * 4}
    \pgfmathsetmacro{\gradyUnrotated}{sin(\angle)/\b * 4}
    
    % Now we rotate the gradient vector by the same \rot_angle
    \coordinate (grad_vec) at (
        {\gradxUnrotated*cos(\rot_angle) - \gradyUnrotated*sin(\rot_angle)}, 
        {\gradxUnrotated*sin(\rot_angle) + \gradyUnrotated*cos(\rot_angle)}
    );
    
    % Draw the *negative* gradient arrow (pointing inward)
    % The node is now placed to the 'right' of the arrow's *endpoint*
    \draw [gradient] (x_point) -- ($(x_point) - (grad_vec)$)
          node[right=2pt, black] at ($(x_point) - (grad_vec)$) {$-\bE_z[\nabla\ell(x,z)]$}; % Label is normal thickness and black

    % --- 4. The Second, Rotated Ellipsoid Set (1 line) ---
    % Draw only the outer ellipsoid (s=1.3)
    \draw [newset, rotate around={-45:(x_point)}] (x_point) ellipse (1.3cm and 1.3*0.4cm);
    
    % Add label for the red ellipsoid set. Placed at the top-left, slightly lower.
    % We use polar coordinates relative to x_point and a small vertical shift.
    \node[anchor=south east, red, inner sep=2pt] at ($(x_point) + (135:1.4) + (0, -0.65)$) {$\Sigma\coloneq\Var_z[\nabla\ell(x,z)]$}; % Label is normal thickness, red, and moved down

    % Placed top-left of the center (0,0) using polar coordinates
    \node[anchor=south east, blue, inner sep=2pt] at($(x_point) + (-25:5.4) + (0, -3)$){$\bE_{z}[\nabla^2 \ell(x,z)]$};
    
    % --- 5. Axes for x1 and x2 ---
    % Position for the origin of the axes (bottom-left)
    \coordinate (axis_origin) at (-5,-3); 

    % Horizontal Axis (x1) - Made 3x larger
    \draw [axis] (axis_origin) -- ($(axis_origin) + (5.75,0)$) node [below] {$x_1$};
    % Vertical Axis (x2) - Made 3x larger
    \draw [axis] (axis_origin) -- ($(axis_origin) + (0,5.75)$) node [left] {$x_2$};

\end{tikzpicture}
\caption{Illustration of model misspecification. The geometry of the expected loss curvature $\nabla^2 f$ differs from the geometry of the gradient noise ($\Sigma$). While setting $P \approx \Sigma^{-1}$ whitens the noise, it may result in unstable updates along high-curvature directions.}
\label{fig:levelsets}
\end{wrapfigure}
At this point, we highlight that the learning problem is governed by three sources of curvature: the Hessian of the population risk $\nabla^2 f \equiv \nabla^2 f(\hat x)$ for some minimiser $\hat x$, the gradient covariance matrix $\Sigma$, and the preconditioner $P$ which is chosen by the practitioner.
The goal of this paper is to understand, in the \emph{finite-sample nonasymptotic setting},
how does the excess risk of \ac{PSGD} depend on the interaction between $\nabla^2 f$, $\Sigma$, and $P$.
While, in the idealised scenario, these quantities coincide \citep{Amari1998Natural},
in the general \emph{misspecified} learning setting where $\Sigma \neq H$,  the disparity creates a fundamental trade-off. This trade-off is addressed in practice in different ways by different optimisation algorithms. Methods like Adam \citep{Kingma2014Adam} and K-FAC \citep{Martens2015Optimizing} target an approximate conditioning $P \approx \Sigma^{-1}$, while others, such as AdaHessian \citep{Yao2021ADAHESSIAN}, PROMISE \citep{Frangella2024PROMISE}, SAPPHIRE \citep{Sun2025SAPPHIRE}, SketchySGD \citep{Frangella2024SketchySGD}, target the inverse of the \emph{expected Hessian} $\nabla^2 f$.
Thus, without the characterisation of the statistical properties associated with the mismatch between these geometries, the choice of preconditioner in the misspecified regime remains largely heuristic, which can lead to undesired behaviour (see \Cref{fig:levelsets} for a graphical example). From a non-asymptotic statistical perspective, here we ask \emph{what is the optimal choice of $P$ with respect to $\nabla^2 f$, and $\Sigma$?}

In this paper we are primarily interested how excess risk $\E[\delta f(x_t)]$ depends on \emph{effective dimension}
\begin{align}
  \label{eq:eff-dim}
  \tr\pr{(\nabla^2 f)^{-1} \Sigma}
\end{align}
which commonly appears in statistics as a replacement for the ambient dimension. This is also known as the \ac{TIC} in the context of information theory~\citep{Shibata1989Statistical}.
For example, the effective dimension controls excess risk bounds of linear (ridge) regression, for exact minimisers~\citep{bach2024learning}, \ac{SGD} with iterative averaging~\citep{neu2018iterate}, as well as asymptotic analysis in stochastic approximation \citep{polyak1992acceleration}.
While it is known that dependence of the excess risk on \eqref{eq:eff-dim} is not improvable asymptotically, we ask here how $P$ interacts with effective dimension in the non-asymptotic regime.

In particular, we will study this question through the lens of \emph{generalisation error} $x \mapsto f(x) - f_S(x)$ and \emph{algorithmic stability}, which is a classical framework dating back to the study of nearest-neighbor rules~\citep{devroye1979distribution} and \ac{ERM} problems \citep{Bousquet2002Stability}.
The stability approach asks whether the solution produced by the learning algorithm is insensitive to small perturbations in the training set, such as the removal of a data point or its replacement by an independent copy.
Namely, if $\hat x^{(i)}$ is a parameter produced with such a perturbation (say when $z_i$ is replaced by its independent copy $z_i'$), then the expected generalisation error is directly linked to stability gauged by the difference of losses:
\begin{align}
  \label{eq:stability}
  \E[f(\hat x) - f_S(\hat x)] = \frac1n \sum_{i=1}^n \E_{S,z_i'}\br{\ell(\hat x^{(i)}, z_i) - \ell(\hat x, z_i)}~.
\end{align}
Numerous works \citep{feldman2019high,Bousquet2020Sharper} establish high-probability bounds on the generalisation error by controlling the \emph{uniform} stability $\sup_{S,z,i} |\ell(\hat x, z) - \ell(\hat x^{(i)}, z)|$ for various \ac{ERM} formulations.
However, such notions of stability tend towards covering the \emph{worst-case} and are not suitable to achieving our goal, since $(\nabla^2 f, \Sigma)$ are distribution-dependent quantities. Here we turn our attention to the weaker notion of \emph{on-average} stability $\max_i \E_{S,z_i'}[\ell(\hat x^{(i)}, z_i) - \ell(\hat x, z_i)]$ which has, so far, largely been used to study \ac{ERM} algorithms instead of the \ac{SGD}-type algorithms of interest in this work~\citep{kearns1997algorithmic,Bousquet2002Stability,elisseeff2005stability}.

Algorithmic stability of \ac{SGD}-type algorithms has been studied extensively over recent years.
A seminal paper from \cite{Hardt2016Train} derived uniform stability bounds for simultaneously smooth Lipschitz and convex loss functions. These proof techniques were later extended by \cite{Kuzborskij2018DataDependent,Lei2023Stability} to the on-average stability, observing that generalisation error can be controlled by the data dependent quantities (such as the empirical risk), leading to optimistic bounds. However, none of these works showed dependence on the effective dimension or preconditioner $P$. 
By targeting a data-dependent analysis of \ac{PSGD}, one runs into several difficulties. Most commonly known, is the difficulty of managing dependence between parameter iterates and the dataset, which is usually circumvented by restricting to the single pass setting. In the present work, we consider the multi-pass setting and we develop methods to manage parameter-dataset dependence.

\subsection{Our contributions}
\begin{enumerate}
\item We develop an \emph{on-average stability} analysis of \emph{multipass} SGD that overcomes the technical challenge of dependence between iterates arising through reused data points --- see \Cref{sec:sketch} for the sketch of the analysis.
\item We derive excess risk bounds for multipass \ac{PSGD} that depend on the effective dimension governed jointly by the loss curvature, preconditioning matrix, and gradient noise.
\item We identify a regime where an improperly chosen preconditioner leads to suboptimal effective dimension dependence in both optimisation and generalisation.
\item We complement our results by obtaining matching instance-dependent lower bounds.
\end{enumerate}
Rather than working directly with $\nabla^2 f$ we employ a proxy \ac{PD} matrix $H$ such that $\nabla^2 \ell \preceq \cstSmooth H$ and perform an analysis in the geometry of the $\|\cdot\|_H$-norm.
We focus on $\cstSmooth$-smooth (but not necessarily Lipschitz) losses and we consider two structural cases: strongly convexity, and non-convex losses satisfying a \ac{PL} condition (see \cref{eq:PL-def}), both in $\|\cdot\|_H$ norm.

\paragraph{Smooth strongly convex losses.}
In the first setting, we consider an arbitrary choice of the preconditioner $P$.
\Cref{cor:risk_bound_Pinv} implies that with step size $\sim 1/(t+1)$ the excess risk satisfies
\begin{align*}
  \bE_{S,\cA}[\delta f(x_t)] \;\le\; \frac{64}{\lmin(PH)\cstScvx} \left( \frac{\bE_S[\tr(PHP\Sigma_S)]}{t+1} \;+\; \tr(P\Sigma) \left( \frac{1}{\sqrt{n(t+1)}}+\frac{1}{n} \right) \right)~.
\end{align*}
Observe that the excess risk depends on the term $\tr(P \, \Sigma)$ which resembles effective dimension and multiplies $1/n$, which is a statistical rate.
The term $\bE_S[\tr(PHP\Sigma_S)]$ bears a similar role as it multiplies $1/t$, which is an optimiser convergence rate.\footnote{Note that $\E[\Sigma_S]$ can be controlled in terms of $\Sigma$ with bias of order $\sqrt{\tr(P H P \Sigma)/n}$, see \Cref{lemma:sigma-s-to-sigma}.} At the same time it is known that the optimal statistical rate is $\tr(H^{-1} \Sigma) / n$ and so the above suggests that the optimal choice $P = H^{-1}$ recovers the optimal rate $\tr(H^{-1} \Sigma) (1/t + 1/n)$, while other choice will lead to the suboptimal statistical rate. This also demonstrates that the geometry required to minimise the variance in the optimisation error is identical to the geometry required to minimise finite-sample algorithmic instability.
Thus, second-order information is not only a tool for speed, but a mechanism for robustness against sampling noise.

The key to presence of $\tr(P \, \Sigma)$ stems from combination of on-average stability analysis and working with weighted Euclidean norms.
This is elucidated by \Cref{lemma:PSGD_gen_stab}, which states that for any stochastic iterative algorithm that satisfies geometric contractivity between $x_t$ and $x_t^{(i)}$ (where the later is obtained on the perturbed training set), for any \ac{PD} matrix $M$ and a constant step size we have
\begin{align*}
  \bE\br{ \| x_t - x_t^{(i)}\|^2_{M} } = \mathcal{O}\pr{ \frac{\tr(P  M  P  \Sigma)}{n}\left(\eta +\frac1n\right) } \qquad \text{as} \qquad n \to \infty.
\end{align*}
Note that this result only requires smoothness, but not convexity of the loss.
First, the lemma captures stability of \ac{PSGD} in a subspace of choice rather than globally.
Choosing curvature $M = P^{-1}$ naturally leads to analysis of preconditioned SGD as iterates live in a subspace spanned by the preconditioner.
Second, working with on-average stability allows us to gain dependence on $\Sigma$, whereas a stronger, uniform stability would be oblivious to geometry of the noise.

\paragraph{On-average stabiliy for smooth \ac{PL} losses.}
Next in \Cref{subsec:pl_stab} extend our analysis to family of non-convex smooth losses that satisfy \ac{PL} condition. In particular, we show that excess risk is controlled by the effective dimension,
\begin{equation*}
    \E[\delta f(x_t(S))] \leq \frac{4 \beta}{\mu} \E[ \delta f_S(x_t(S)) ] + \frac{8 \tr(H^{-1} \Sigma)}{\mu n}.
  \end{equation*}
for large enough $n$. Note that, excess risk no longer depends on a particular $P$ and behaves as if an optimal $P$ was chosen. The expected optimisation error $\E[\delta f_S(x_t(S))]$ scales with the effective dimension as well and the bounded is given by the standard convergence analysis for \ac{PL} objectives~\citep{karimi2016linear}.

\paragraph{Lower bounds.}
Finally, a natural question is whether the results we presented are optimal. On one hand it is known that from both statistical and optimisation perspectives dependence on $\tr(H^{-1} \Sigma)$ is optimal as there exist asymptotic lower bounds (Cramér-Rao type lower bound \citep{polyak1992acceleration}). To this end, focusing on the strongly-convex model, in \Cref{subsec:lower_bounds} we complement this fact in non-asymptotic sense, by showing that in minimax lower bounds on the excess risk are of order $\tr(H^{-1}, \Sigma) / (n \, \cstSmooth)$. Clearly we cannot expect any improvement in the minimax sense, however, the message our analysis conveys is that a bad choice of the preconditioner might lead to a poor statistical performance, and so minimax analysis is no longer appropriate. To this end we present an \emph{instance-dependent} lower bound, albeit limited to a single pass \ac{PSGD}. In particular, for a decaying step size $\eta_t \sim 1/t$, for a sufficiently large $t$, the expected excess risk behaves as
\begin{align*} 
	\frac{\tr(PHP\Sigma)}{\lmax(PH)\lmin(PH)} \cdot \frac1t~.
\end{align*}
While for the optimal choice of the preconditioner $P = H^{-1}$ this bound matches the upper bound, for a badly chosen preconditioner $P$ (for instance, we can construct $P$ that approaches rank-deficiency) the above our result implies that $\tr(H \Sigma) / (\varepsilon t)$ with $t > 4 / \varepsilon$. In other words, for a general curvature $(\Sigma, H)$ and large $t$, the associated constant in front of the asymptotic rate of the excess risk can be arbitrarily large, even with decaying step. This, once more, emphasises the impact of the preconditioning on statistical performance.

\paragraph{Notation and terminology.}
For symmetric matrices $A, B$, we write
$A \preceq B$ to denote the semidefinite order, meaning that $B-A$ is \ac{PSD}, and similarly $\prec$ to denote \ac{PD}. We denote $\| x \|_{H} = x^\top H x$ for a positive definite matrix $H\succ 0$. We let $\lmin(A)$ and $\lmax(A)$ denote the smallest and the largest eigenvalue and $\cond(A) = \lmax(A)/\lmin(A)$ is the condition number of a matrix $A\in\bR^{d\times d}$. For $\cstScvx$-strongly convex $\cstSmooth$-smooth function w.r.t.\ $\| \cdot\|_H$-norm we denote $\cond_\ell = \cstSmooth/\cstScvx$. Roman font $\rQ, \rP_x$ denote probability distributions, the latter parameterised by the vector $x$.

\section{Proof Sketch and Technical Challenges}%
\label{sec:sketch}
The expected excess risk of an estimator $\hat{x}$ is typically bounded by balancing the trade-off between error terms originating from the generalisation component and those arising from offline optimisation of the empirical risk:
\begin{equation*}
     \bE_S\left[\delta f(\hat x)\right] = \underbrace{\bE_S\left[ f(\hat{x}) - f_S(\hat x) \right]}_{\text{generalisation}} + \underbrace{\bE_S\left[ f_S(\hat{x}) - f_S(\tilde{x}) \right]}_{\text{optimisation}}~,
\end{equation*}
where $\tilde x =\argmin_{x\in\cX} f(x)$. Here the optimisation error can be further upper bounded using the \ac{ERM} $x^\ast_S \in \argmin_{\cX} f_S(x)$ and noting that $\bE_S\left[ f_S(\hat{x}) - f_S(\tilde{x}) \right] \leq \bE_S\left[ f_S(\hat{x}) - f_S(x_S^\ast) \right]$.  

The generalisation term can be controlled using the standard algorithmic stability argument. Let $\hat{x}^{(i)}$ be computed from a perturbed dataset $\Si = S\setminus\{z_i\}\cup \{ z' \}$, where $z' \sim \rQ$ with the same algorithmic procedure as $\hat{x}$. Then using the standard symmetricity argument
leads to observation that the generalisation term is equal to the \emph{on-average algorithmic stability}, \cref{eq:stability}.

\subsection{Generalisation Geometry via On-Average Multipass Stability with Correlated Iterates}
In the multi-pass setting, when $x_t$ is computed by sampling examples from $S$ with replacement, the iterate is not independent with previously seen samples $z_{i_t}$ and the standard stability analysis fails. This usually forces the analysis to rely on \emph{uniform stability bounds} \citep{Hardt2016Train} that assume uniform $\ell(\cdot, z)$ is $L$-Lipschitz for all samples
\begin{equation*}
	\bE_{z\sim\rQ} [ |\ell(\xtS,z) - \ell(\xtSi,z) |] \leq \sup_{z\in\cZ}  |\ell(\xtS,z) - \ell(\xtSi,z) |  \leq L \| \xtS - \xtSi\|.
\end{equation*}
This step effectively removes any dependence on the data distribution and its interaction with finer geometric properties of the loss, which is commonly pointed out as a limitation~\citep{Zhang2017Understanding}. 

In order to reveal the generalisation geometry, we exploit that $\ell(\cdot, z)$ is $\cstSmooth$-smooth w.r.t $\| \cdot \|_H$-norm, and show that, when $\eta_t$ is small, the generalisation is governed by
\begin{align*}
  &\bE_S\left[ f(\xtS) - f_S(\xtS) \right]\\
  &\qquad= \mathcal{O}\left( \Var_{z\sim\rQ}\br{\| \nabla \ell(\xtS,z) \|^2_\ast}^{1/2} \cdot \bE_{\cA, S, z'}\br{\|x_t(S) - x_t(S^{(i)}) \|^2}^{1/2} \right).
\end{align*}
The choice for $\| \cdot \|$-norm controlling the squared parameter stability $\stabP^2(\xtS, \| \cdot\|) \coloneq \bE_{\cA, S, z'}[\|x_t(S) - x_t(S^{(i)}) \|^2]$ plays a crucial role in two ways: it bounds the parameter stability and its dual norm will interact with the the noise of gradients.
We restrict ourselves to Hilbert spaces and consider $\stabP^2 (\xtS, \| \cdot\|_M)$ for some $M\succ 0$. If the deterministic PGD update is $r$-contractive we can upper bound the parameter stability as
\begin{align*}
	\stabP^2( x_{t+1} (S)) &\leq (1 - c_1 \eta_t \cstCon) \, \stabP^2(\xtS ) + c_2 \, \frac{\eta_t^2}{n} \, \bE\br{\|P(\xi_t - \tilde{\xi_t}) \|^2_M}~,
\end{align*}
where $\xi_t \coloneq \nabla \ell(x_t(S), z_i) - \nabla \ell (x_t(\Si), z')$ involves the challenging term with the correlated samples and parameters, and $\tilde \xi_t \coloneq \nabla f(x_t(S)) - \nabla f(x_t(\Si))$. We overcome the \emph{problem of correlated iterates in the multi-pass setting} by being able to upper bound it as
\begin{align*}
 \bE\br{\| \xi_t - \tilde{\xi_t} \|^2_{PMP}} \leq \tr(PMP\Sigma) + c_3 \cstSmooth^2 \stabP(\xtS)~.
\end{align*}
We identify a condition $n \geq \cond_\ell \cond(PH)$ depending on the geometry of $\ell$ and $P$, that ensures the contribution of the correlated terms is benign, resulting in
\begin{equation*}
  \stabP^2(\xtS, \| \cdot \|_M) \;\leq\; \underbrace{c_4 \, \frac{\tr(PMP\Sigma)}{n^2}}_{\text{Irreducible Fast Rate}} \;+\; \underbrace{c_5 \, \eta\, \frac{ \tr(PMP\Sigma)}{n}}_{\text{Optimisation Variance}}\qquad \text{for a fixed $\eta_t = \eta$.}
\end{equation*}
This decomposition is sharper than $\mathcal{O}(1/n)$ providing finer control as $\mathcal{O}(1/n^2)$ when $\eta \leq 1/n$, it isolates the intrinsic statistical complexity (the fast rate) from the noise induced by the algorithm's step size, and establishes \emph{on-average} stability for arbitrary $t$ provided $n$ is large enough. Thus, when the deterministic PGD update is $r$-contractive in $\| \cdot \|_M$-norm, $\eta$ is small enough\footnote{Or, for example, in the standard setting of $\eta_t \approx 1/t$ and $t\geq n$.}
\begin{equation*}
	\bE_S\left[ f(\hat{x}) - f_S(\hat x) \right] = \mathcal{O}\left(\frac{\sqrt{\tr(M^{-1}\Sigma) \tr(PMP\Sigma)}}{n}\right) 
\end{equation*}
and selecting $\| \cdot \|_M$ determines the analysis's sensitivity to parameter stability and  gradient noise.

\subsection{Spectral Alignment under Geometric Mismatch}
In the standard stability analysis, for $\cstScvx$-strongly convex and $\cstSmooth$-smooth functions w.r.t.\ $\| \cdot\|_2$, the contractivity of the gradient descent (without preconditioning) comes from \emph{gradient co-coercivity}:
\begin{equation*}
	\inner{\nabla f(x) - \nabla f(y)}{x-y} \geq \Big(\frac{\cstScvx\cstSmooth}{\cstScvx+\cstSmooth} \|x-y\|_2^2 + \frac{1}{\cstScvx+\cstSmooth} \|\nabla f(x) - \nabla f(y)\|_{2}^2 \Big).
\end{equation*}
When the geometry of $\ell$ is defined by $\| \cdot \|_H$-norm and we use preconditioning $P$ the term needed to be bounded is $\inner{HP (\nabla f(x) - \nabla f(y))}{x-y}$, which does not need to be positive unless $P = H^{-1}$. However, in practical settings we have almost never that $P=H^{-1}$, the matrices $P,H$ are misaligned and do not commute. 

We introduce a rigorous condition for \emph{spectral alignment} based on the matrix pencil $(P, H^{-1})$ and establish a \emph{generalised co-coercivity inequality} for gradients under non-commuting preconditioning
\begin{align*}
	\inner{HP (\nabla f(x) -\nabla f(y))}{x-y} &\geq \frac{\lmin(P H) \, C_{\ell,P}}{\cstScvx+\cstSmooth}\Big(\cstScvx\cstSmooth \|x-y\|_H^2 + \|\nabla f(x) - \nabla f(y)\|_{H^{-1}}^2 \Big),
\end{align*}
where the constant $\cstSA \in (0, 1]$ tracks the quality of the alignment: $\cstSA = 1$ for quadratic functions ($\cstSmooth = \cstScvx$) and $\cstSA \to 0$ for badly aligned problems. This property allows to show the contractivity of the preconditioned gradient update in the parameterised family of metrics $\| \cdot\|_{M_\theta}$ defined by $M_\theta = H^{\frac{1}{2}(1-\theta)} P^{-\theta} H^{\frac{1}{2}(1-\theta)}$ interpolating between: the natural metric of the problem when $\theta = 0$ ($\| \cdot \|_H$) for $P,H$ are spectrally aligned, and the metric defined by the algorithm when $\theta = 1$ ($\| \cdot \|_{P^{-1}}$), but which holds for any $P\succ 0$. 

\section{Preliminaries}
\paragraph{Relative smoothness \& strong convexity.} We define the geometry w.r.t\ the weighted norm $\|\cdot \|_H$.
\begin{definition}[Smoothness w.r.t.\ $\|\cdot\|_H$]\label[definition]{def:smooth}
Let $H \succ 0$ such that $\lmax(H)=1$, and $\cstSmooth > 0$. The function $f(x)$ is $\cstSmooth$-smooth w.r.t $\| \cdot \|_H$ when
$
    f(y) - f(x) \leq \inner{\nabla f(x)}{y-x} + \frac\cstSmooth2\| y-x\|^2_H
$
or equivalently, $\| \nabla f(x) - \nabla f(y)\|_{H^{-1}} \leq \cstSmooth \| x-y\|_{H}$ for convex $f$.
\end{definition}

\begin{definition}[Strong convexity w.r.t.\ $\|\cdot\|_H$]\label[definition]{def:scvx}
Let $H \succ 0$ such that $\lmax(H)=1$, and $\cstSmooth > 0$. The function $f(x)$ is $\cstScvx$-strongly convex w.r.t $\| \cdot \|_H$ when
$
    f(y) - f(x) \geq \inner{\nabla f(x)}{y-x} + \frac\cstScvx2\| y-x\|^2_H.
$
\end{definition}
The definitions are special cases of \emph{relative smoothness and strong convexity}, see \citep[Definition 1.1 and 1.2]{Lu2018Relatively}, where we choose the \emph{reference function} to be $h(x) = \ip{x, H x}$. They have been referred to also as ``matrix smoothness'' employed by \citep{Thomas2020Interplay, Li2024DetCGD}. We denote the condition number of the loss w.r.t.\ the $\| \cdot \|_H$-norm geometry by $\cond_\ell \coloneq \cstSmooth/\cstScvx$.

Since $\cond_\ell$ expresses the discrepancy between $f$ and a quadratic function, on a bounded domain it can be bounded using higher order smoothness. If $\ell(\cdot,z)$ has $\gamma$-smooth Hessian, i.e., $\lmax(\nabla^2 \ell(x_1,z) -\nabla^2 \ell(x_2,z)) \leq \gamma \| x-y\|_2$, we have the following bound
\begin{equation*}
	\cond_\ell \leq \frac{1+\gamma R / \lmin(H)}{1-\gamma R/ \lmin(H)},\quad\text{for $\| x-x_0 \|_2\leq R$}.
\end{equation*}
For $f\in C^2$, the combination of \Cref{def:smooth,def:scvx} acts as a quadratic upper and lower bound respectively: $\cstScvx H \preceq \nabla^2 \ell(x,z) \preceq \cstSmooth H$.

\paragraph{Generalised co-coercivity.}
The following defines the spectrally aligned preconditioner when the relative condition number $\cond(PH)$ is sufficiently bounded compared to $\cond_\ell$.
\begin{definition}[Spectrally aligned preconditioner]\label[definition]{def:relcond_bound}
For $\ell(\cdot,z)$ that is $\cstScvx$-strongly convex and $\cstSmooth$-smooth w.r.t\ $\| \cdot \|_H$, we say that $P$ is $\cstSA$-\emph{spectrally aligned} with the geometry of $\ell(\cdot,z)$ iff
\begin{equation*}
	\cond(PH) \leq \rho_\ell^2\quad \text{with}\quad\cstSA = \frac{\rho_\ell^2 - \cond(PH)}{\rho_\ell^2 - 1}\quad \text{and}\quad \rho_\ell \coloneq \frac{\sqrt{\cond_\ell} + 1}{\sqrt{\cond_\ell} - 1} > 1.
\end{equation*}
\end{definition}
This decomposes the conditioning misalignment into two parts: $\rho_\ell$ reflects how well the model of relative smoothness/strong convexity captures the actual geometry of $\ell$, and $\cond(PH)$ reflects how well the algorithm, i.e., the choice of $P$, captures the model curvature defined by $H$.
\Cref{def:relcond_bound} is satisfied for many widely used choices of $P$ and allows for a fine description of the geometry needed for generalisation.
\begin{example}[Inexact-Newton Methods ($q$-approximate inverse curvature)]\label{ex:q_approx}
Assume that for some $q\ge 1$ the preconditioner $P\succ 0$ satisfies $\frac{1}{q}\,H^{-1} \;\preceq\; P \;\preceq\; q\,H^{-1}$. Then $(1/q)I \preceq PH \preceq q I$, hence $\cond(PH)\le q^2$. Therefore, whenever $q^2<\rho_\ell^2$,
Assumption~\ref{def:relcond_bound} holds and the alignment constant is lower bounded as
\begin{equation*}
    C_{\ell,P}
    \;=\;
    \frac{\rho_\ell^2-\cond(PH)}{\rho_\ell^2-1}
    \;\ge\;
    \frac{\rho_\ell^2-q^2}{\rho_\ell^2-1}.
\end{equation*}
Such uniform spectral boundedness assumptions have been used in the Quasi-Newton literature to prove global convergence; see, e.g., see \citep[Sec.~3.3]{Nocedal2006Numerical} and \citep{DennisJr1977QuasiNewton}. More recently \citet{Cheng2010Spectral}, showed that ensuring the $q$-approximate inverse curvature leads to improved numerical performance.
\end{example}
\begin{example}[Diagonal preconditioning]\label{ex:jacobi_dd}
Let $P$ be a diagonal preconditioner $P \coloneq \diag(H)^{-1}$.
If $A\coloneq D^{-1/2} H D^{-1/2}$ is strictly diagonally dominant in the sense that $\alpha \;\coloneq\; \max_i \sum_{j\neq i} |A_{ij}| \;<\; 1$, then by Gershgorin disc theorem we have that the spectrum $\lambda(A)\subset[1-\alpha,\,1+\alpha]$, hence $\cond(PH)=\cond(A)\;\le\;\frac{1+\alpha}{1-\alpha}$. Consequently, whenever $\frac{1+\alpha}{1-\alpha}<\rho_\ell^2$,
\begin{equation*}
    C_{\ell,P}
    \;\ge\;
    \frac{\rho_\ell^2-\frac{1+\alpha}{1-\alpha}}{\rho_\ell^2-1},
\end{equation*}
yielding a simple explicit $C_{\ell,P}$ bound for diagonal preconditioning whenever $H$ is close to diagonal.
\end{example}
\Cref{def:relcond_bound} allows to derive a generalisation of the standard gradient co-coercivity result, e.g., see \citep[Theorem 2.1.12]{Nesterov2018Lectures}, that applies to preconditioned gradients and specific case of relative smoothness and strong convexity \citep{Lu2018Relatively}.
\begin{lemma}[Co-coercivity of spectrally aligned PSGD updates]\label[lemma]{lemma:cocoercivity}
Let $f$ be $\cstScvx$-strongly convex and $\cstSmooth$-smooth w.r.t.\ $\| \cdot \|_H$ and $P$ is $\cstSA$-spectrally aligned with $\ell(\cdot,z)$, i.e., $\cond(PH) < \rho_\ell^2$ in \Cref{def:relcond_bound}. Then for all $x, y \in \bR^d$:
\begin{align*}
    \langle \nabla & f(x) - \nabla f(y), HP(x-y) \rangle \geq \frac{\lmin(PH) C_{\ell,P}}{\cstScvx + \cstSmooth}\pr{\cstScvx\cstSmooth \|x-y\|_H^2 +  \|\nabla f(x) - \nabla f(y)\|_{H^{-1}}^2 }.
\end{align*}
\end{lemma}
Proof is given in \Cref{subsec:proof_cocoercivity}. For $P = H^{-1}$ this recovers the standard co-coercivity of gradients. Note, that \Cref{lemma:cocoercivity} does not require that $P$ and $H$ commute.

%%% Local Variables:
%%% mode: latex
%%% TeX-master: "arxiv_paper"
%%% End:

\section{Excess risk bounds of PSGD via on-average stability}\label{sec:risk_convex_PSGD}
In this section, we derive excess risk bounds for the PSGD algorithm via on-average stability. Throughout this section, assume the following:
\begin{assumption}\label{asm:sm-sc}
Suppose that for each $z \in \mathcal{Z}$, $\ell(\cdot, z)$ is $\cstSmooth$-smooth with respect to the norm, $\|\cdot\|_H$.
\end{assumption}

\begin{assumption}\label{asm:cov}
Suppose there exists $\Sigma \succ 0$ such that $\operatorname{Cov}_{z \sim Q}(\nabla \ell(x, z)) \preccurlyeq \Sigma$ for all $x\in\cX$.
\end{assumption}
In traditional stabiltiy analyses of SGD-type algorithms, a uniform Lipschitz assumption is typically employed to relate algorithmic stability to parameter stability. However, this Lipschitz assumption both rules out several settings of interest (e.g. strongly convex losses on non-compact domains) and often conceals the curvature information present in smoothness and convexity. To fully exploit the geometry of the problem, we proceed without the global Lipschitz assumption. We first provide a general stability result for an algorithm $\mathcal{A}$ that maps from $\mathcal{Z}^n$ to a random variable on $\cX$.
\begin{lemma}\lemmalabel{lemma:risk_decomposition}
Suppose that assumptions \ref{asm:sm-sc} and 
\ref{asm:cov} hold and let $M \succ 0$. If the algorithm $\mathcal{A}$ is $L^2$- on-average parameter stable in $\|\cdot\|_M$-norm with constant $\varepsilon_{\mathrm{pstab}}^2 \geq 0$, then the expected excess risk on the parameters $x = \mathcal{A}(S)$ satisfies,
\begin{equation*}
	\bE_{S,\cA} [\delta f(x_t)] \;\le\; 2 \bE_{S,\cA}[\delta f_S(x_t)] + \frac{\tr(M^{-1} \Sigma)^{1/2}}{2} \, \stabP + 4 \cstSmooth \lmax(H M^{-1}) \, \stabP^2.
\end{equation*}
\end{lemma}
Note that the primary limitation borne from the Lipschitz assumption, that we are able to overcome, is that parameter stability is measured in the weaker $\|\cdot\|_M$, whereas our analysis uses $\| \cdot\|_M^2$-norm instead. While this requires a tighter control on the iterates, it allows to use smoothness to identify the explicit role of the curvature (via $H$ and $M$) in the generalisation bound. The matrix $M$ is chosen according to its amenability to the parameter stability analysis, but to optimise the bound, $M$ must also be chosen to align with either the curvature $H$ or the covariance matrix $\Sigma$. Thus, under misspecification, the natural geometry to analyse parameter stability in is not immediate.

We now focus on the PSGD algorithm defined by a positive definite preconditioner, $P\succ 0$. To analyse the stabiltiy PSGD, we will obtain that if the update is contractive in a suitable geometry
\begin{equation}\label{eq:contraction}
	\| x - \eta P \nabla \ell(x, z) - y + \eta P \nabla \ell(y, z) \|^2_M \leq (1- \eta \cstCon)\| x - y \|^2_M.
\end{equation}
\begin{lemma}[On-average parameter stability of PSGD]\lemmalabel{lemma:PSGD_gen_stab}
Suppose that Assumption \ref{asm:sm-sc} holds, choose any matrix $M\succ 0$ and constants $\bar{\eta}, r > 0$ such that for any $x, y \in \cX, z \in \mathcal{Z}$ and $\eta \leq \bar{\eta}$, the $r$-contractivity property in \cref{eq:contraction} holds. Then, if $\sup_s \eta_s \leq \bar{\eta} \wedge r^{-1}$ and $n \geq 8 \cstSmooth \sqrt{\lmax(HPMP)}$ $\cdot \sqrt{\lmax(M^{-1}H)} / \cstCon$, we have that $\mathcal{A}_{P, t}$ is on-average parameter stable with constant,    
\begin{equation*}
	\stabP^2 \leq 64 \bigg ( \frac{\bar{\eta}_t}{8n} + \frac{1 - e^{-T_t r/4}}{n^2 \cstCon^2} \bigg ) \tr(PMP\Sigma)~,
\end{equation*}
where $T_s = \sum_{s'=0}^{s-1} \eta_{s'}$ and $\bar{\eta}_t = \sum_{s < t} e^{- r\frac{T_t - T_s}{4}} \eta_s^2$.
\end{lemma} 
The proof is provided in \cref{subsec:proof_PSGD_gen_stab}. A significant advantage of this bound over those in \citep{Hardt2016Train} is the explicit dependence on the data distribution via the trace term $\tr(PMP\Sigma)$. Furthermore, unlike \citep{Kuzborskij2018DataDependent}, our result captures the exact interaction between the curvature $H$, the preconditioner $P$, and the noise $\Sigma$, while valid in the multi-pass setting.

The quantity $\bar{\eta}_t$ characterises how memory of past step-sizes decays. For standard step-size schedules, it behaves intuitively: for example, with a linearly decaying step size $\eta_t = c/t$, we have $\bar{\eta}_t \leq c^2/t$. 

\subsection{On-average stability and risk bounds for strongly convex smooth losses\label{subsec:risk_PSGDgeneral}}
Under additional assumption that the loss is also $\cstScvx$-strongly convex, we can show that the PGD update is $r$-contractive in a specific family of $\| \cdot \|_M$-norms.
\begin{assumption}\label{asm:Scvx}
Suppose that for each $z \in \mathcal{Z}$, $\ell(\cdot, z)$ is $\cstScvx$-strongly convex with respect to $\|\cdot\|_H$.
\end{assumption}
\Cref{lemma:contract_pgd_Mtheta} in \Cref{sec:rel_coercivity} shows that under \Cref{asm:Scvx}, the PGD update is contractive in $\| \cdot \|_{M_\theta}$ where $M_\theta \;\coloneq\; H^{\frac12(1-\theta)} P^{-\theta} H^{\frac12(1-\theta)}$ for $\theta\in [0,1]$ interpolating between $H$ and $P^{-1}$.

By combining the stability result in \Cref{lemma:PSGD_gen_stab}, the contractivity result in \Cref{lemma:contract_pgd_Mtheta}, and the optimisation rates for PSGD (see \Cref{sec:opt_bounds}), we can derive explicit generalisation bounds, see \Cref{lemma:risk_bound_Mtheta}. We consider two natural geometries for measuring convergence: the geometry induced by $P^{-1}$ ($\theta = 1$) and the geometry induced by the Hessian $H$ ($\theta = 0$).
\begin{proposition}[Risk bounds in geometry defined by $P^{-1}$)]\proplabel{cor:risk_bound_Pinv}
  Let $P\succ0$, suppose that Assumptions~\cref{asm:sm-sc}, \ref{asm:cov} and \ref{asm:Scvx} hold, and that $ n \;\ge\; 4 \cond_\ell \cond(PH)$. Let $ r\;\coloneq\; 2\,\lmin(PH)\, \frac{\cstSmooth\,\cstScvx}{\cstScvx+\cstSmooth}$.
  Let $\Var_z[\nabla \ell(x, z_{i_t})]\preceq \Sigma_S$ for all $x$.
  
If the stepsizes are chosen as $ \eta_t \;\coloneq\; \min\{ 1/(\cstSmooth\,\lmax(PH)), 8/(r(t+1)) \}$, then, for all $t$ sufficiently large, the population excess risk satisfies
\[
\bE_{S,\cA}[\delta f(x_t)] \;\le\; \frac{64}{r}\left( \frac{\bE_S[\tr(PHP\Sigma_S)]}{t+1} \;+\; \tr(P\Sigma) \left( \frac{1}{\sqrt{n(t+1)}}+\frac{1}{n} \right) \right)~.
\]
\end{proposition}
We get sublinear $\bigO(1/t + 1/\sqrt{t\, n}+ 1/n)$ convergence rate which matches the single pass (when $n= t$) result \citep{Rakhlin2012Making}, however with the precise rates depending on the interplay of the curvature, variance of noise, and how well the preconditioning adapts to these.
Note that $\E[\Sigma_S]$ can be bounded by $\Sigma$ with additive bias of order $L \cstSmooth \varepsilon_{\text{pstab}}$ assuming that $\ell$ is $L$-Lipschitz, see \Cref{lemma:sigma-s-to-sigma}.

This means that any $P\succ 0$, the excess risk of the last iterate of PSGD converges to zero asymptotically, although the rate in the upper bound can become arbitrarily loose with large $\cond(PH)$, even if the variance is bounded. \Cref{coro:optimalP} shows that minimizing the upper bound in terms of $P\succ 0$ yields that $P = H^{-1}$ minimizes the expected risk and gives the optimal Takeuchi Information Critertion for the noisy strongly convex smooth model (\cref{thm:lower_bound_scvx}).

\begin{remark}[Approximate NGD under misspecification]
Due to the connection to the natural gradient descent discussed in \Cref{sec:related}, this result has implications for NGD under misspecification. Let $\ell(x,z)\coloneq -\log p(z|x)$ be the negative log-likelihood of the distribution of $z\sim P_{x}$ and $\ell(\cdot,z)$ is $\cstScvx$ strongly convex, $\cstSmooth$ smooth w.r.t.\ $\| \cdot \|_H$. If the data distribution differs from the model family (misspecification), we have $\Sigma \neq H$. Our bounds show that choosing $P = H^{-1} \approx (F_{\rP_x}(x))^{-1}$ achieves a generalisation bound that is optimal even under this misspecification.
\end{remark}
In the case where $P$ and $H^{-1}$ are spectrally aligned, we can get more precise bounds through an analysis in $\|\cdot\|_H$-norm. 
\begin{proposition}[Risk bounds in geometry defined by $H$] \proplabel{cor:risk_bound_H}
  Suppose that Assumptions~\Cref{asm:sm-sc}, \ref{asm:cov} and \ref{asm:Scvx} hold, and that $ n \;\ge\; \frac{8\,\cstSmooth}{\cstCon} \sqrt{\lmax(HPHP)}$. Assume further that $\cond(PH)\le \rho_\ell^2$ and let $ r\;\coloneq\; 2\,\lmin(PH)\,\cstSA (\cstSmooth\,\cstScvx)/(\cstScvx+\cstSmooth)$.

  If the stepsizes are chosen as $\eta_t \;\coloneq\; \min\{ \cstSA / (\cstSmooth\,\lmax(PH)\,\cond(PH)), \; 8/(r(t+1)) \}$, then, for all $t$ sufficiently large, the population excess risk satisfies
\[
\bE_{S,\cA}[\delta f(x_t)] \;\le\; \frac{64}{r}\left( \frac{\bE_S[\tr(PHP\Sigma_S)]}{t+1} \;+\; \sqrt{\tr(H^{-1}\Sigma)\,\tr(PHP\Sigma)} \left( \frac{1}{\sqrt{n(t+1)}}+\frac{1}{n} \right) \right)~.
\]
\end{proposition}
Note, that since $\lmax(P) = \lmax(H) = 1$, we have that $\lmax(PH)\leq 1$, and thus $\tr(PHP\Sigma)\leq \tr(P\Sigma)$ making the rate in \Cref{cor:risk_bound_H} less or equal compared the one in \Cref{cor:risk_bound_Pinv}.
%%% Local Variables:
%%% mode: latex
%%% TeX-master: "arxiv_paper"
%%% End:
\subsection{Risk bounds for non-convex losses under PL-property} \label{subsec:pl_stab}

While the above analysis captures the generalisation properties of PSGD along the trajectory, it fails to capture what occurs at convergence. This can be seen due to the fact that, irrespective of the choice of preconditioner, the PSGD iterates should converge to the same empirical risk minimiser, and thus, exhibit the same generalisation properties. The inability of this type of stability analysis to capture generalisation at convergence is known \citep{Hardt2016Train}. For that reason, we turn instead to a black-box analysis of any algorithm  $\mathcal{A}$ that produces parameters that approximately minimise $f_S$.

Here we also consider a more general setting than the previous analysis under strong convexity.
In addition to $\beta$-smoothness w.r.t.\ $\|\cdot\|_H$ we will assume that empirical risk satisfies the following \ac{PL} condition~\citep{karimi2016linear}:
There exists $\mu > 0$ a minimizer $x^\ast$ of $f_S$ such that for all $S, x$,
\begin{align}
  \label{eq:PL-def}
  \frac12 \|\nabla f_S(x)\|_H^2 \geq \mu (f_S(x) - f_S(x^\ast))~.
\end{align}
Our analysis is inspired by that of \cite{Charles2018Stability} and we make the following assumption which is identical to their Assumption 1 and the stability analysis that follows is similar to their proof technique of Theorem 3(iii).
\begin{assumption}\label{ass:erm_proj}
  The empirical risk minimizers for $f_S$ and $f_{\Si}$, i.e., $\hat{x}^\ast, \hat{y}^\ast$, satisfy $\Proj_S(\hat{y}^\ast) = \hat{x}^\ast$, where $\Proj_S$ is the projection on the set of empirical risk minimizers of $f_S$. 
\end{assumption}

\begin{proposition}[Excess risk bounds for PL-losses]\label{lemma:risk_PL}
Suppose that for each $z$, $f_S$ is $\cstSmooth$-smooth and satisfies $\cstPL$-PL property w.r.t.\ $\|\cdot\|_H$ and suppose that Assumption \ref{ass:erm_proj} holds. Then whenever $n \geq 32 \beta \lmax(H \Sigma^{-1})$, we have the excess risk bound,
\begin{equation*}
    \bE_{\cA, S}[\delta f(x_t(S))] \leq  \frac{2\beta}{\cstPL} \bE[\delta f_S(x_t(S))] + \frac{2\tr(H^{-1} \Sigma)}{\cstPL n} + 64 \cstSmooth \frac{\tr(H^{-1} \Sigma)}{\cstPL^2 n^2}.
\end{equation*}
\end{proposition}
Together, these results suggest that the generalisation dynamics are governed by a delicate trade-off mediated by the preconditioner. The choice of $P$ dictates the learning trajectory in two distinct ways:
\begin{enumerate}
    \item \textbf{Optimisation Rate:} $P$ determines the convergence speed of the empirical error $\E[\delta f_S(x_t)]$, primarily through the condition number $\cond(PH)$.
    \item \textbf{Effective Dimension:} $P$ shapes the effective noise geometry, scaling the stability error by $\tr(PHP\Sigma)$ or $\tr(P\Sigma)$ in the worst case.
\end{enumerate}
Furthermore, once the algorithm converges, the excess risk it produces becomes independent of the choice of preconditioner. The optimal choice $P \approx H^{-1}$ simultaneously maximises the convergence rate and minimises the effective dimension, acting as a benefit to both optimisation and generalisation.

%%% Local Variables:
%%% mode: latex
%%% TeX-master: "colt_paper"
%%% End:
\section{Lower bounds on the expected risk\label{subsec:lower_bounds}}
\begin{theorem}[Lower bound]\label{thm:lower_bound_scvx}
Let $\ell(x,z):\cX\times \cZ\rightarrow \bR$ be strongly convex w.r.t.\ $\| \cdot \|_H$ norm in the parameters $x$ and $\cP$ is a family of distributions such that $\forall P\in\cP, x\in\cX$ we have $\bE_{z\sim P}[\nabla \ell (x,z)] = 0$ and $\Var_{z\sim P}(\nabla \ell (x,z)) = \Sigma$. Then we have that the expected excess risk of an estimator computed from $S \sim P$ is lower bounded as
\begin{align*}
\inf_{\hat{x}\in \cX} \sup_{P\in \cP} \bE_{S\sim P^n}[ \delta f(\hat{x}(S)) ] \geq \frac{0.14}{ n\cstScvx} \, \tr(H^{-1}\Sigma)~.
\end{align*}
\end{theorem}
Proof is in \Cref{subsec:proof_lower_bound_scvx}. Theorem \ref{thm:lower_bound_scvx} establishes that the fundamental statistical limit of the problem is governed by the interaction between the geometry of the loss ($H$) and the noise structure (via $\Sigma$).
\paragraph{Algorithmic lower bounds of single pass PSGD}
While previously we demonstrated that the choice of $P = H^{-1}$ results in non-asymptotically optimal rate, here we show that, even in our simple setting, choosing a bad preconditioner $P$ can increase the risk of the last iterate by a multiplicative factor $\cond(PH)$. These bounds can be compared with the lower bound in \citep[Theorem 2.1.13]{Nesterov2018Lectures}, but here we have preconditioning, decaying step-sizes, and lower bound single pass risk. The first result shows that the rate in \Cref{cor:risk_bound_H} is tight up to the constant $\cond(PH)$ for $t$ large enough.
\begin{lemma}[Algorithmic single-pass lower bound]\lemmalabel{lemma:algo_lower}
	The expected excess risk of the online PSGD with $\eta_t = \min\{ 1/\lmax(PH), 2/(\lmin(PH)t)\}$ on the quadratic noisy model is lower bounded as
	\begin{equation*}
          \bE_{z^1,\ldots, z^t} [\delta f(x^{t+1})] \geq \frac{\tr(PHP\Sigma)}{\lmax(PH)\lmin(PH)} \cdot \frac1{t}
          \qquad\quad \,\text{for\,\,} t \geq t_0 \coloneq \floor{2\kappa(PH)}~.
	\end{equation*}	
\end{lemma}
Proof is in \Cref{subsec:proof_algo_lower}. Similar, but more refined bound is derived in \citep[Theorem 5]{Martens2020New}.

For any given $H, \Sigma$, choosing a badly conditioned $P$ can make the risk lower bound arbitrarily larger than the optimal rate.
\begin{corollary}[Algorithmic lower bound for ill-conditioned $P$]\corolabel{coro:algo_lower_badP}
  Choose $\varepsilon > 0$ and assume that $t > 4 / \varepsilon$.
  Let $H, \Sigma\succ0$ and $\lmax(H) = 1$ and $Q$ be the eigenbasis of $H$. Then PSGD with a decaying stepsize $\eta = \min\{1/\cstSmooth\lmax(PH), 2/(t\lmin(PH)\cstScvx) \}$, a preconditioner $P_\varepsilon = I - (1-\varepsilon)q_k q_k^\top$, where the choice of $k$ is explicitly defined by the spectrum of $H$ and $\Sigma$, has the risk lower bounded as
  \begin{equation*}
    \bE_{z^1,\ldots, z^t} [\delta f(x^{t+1})] \geq \left(1-\frac1d\right) \cdot \frac{\tr(H \Sigma) }{\varepsilon t}~.
  \end{equation*}    
\end{corollary}      
The proof is given in \Cref{subsec:proof_algo_lower_badP}. This shows that for a general $H,\Sigma$ and $t$ large enough, the constant in front of the excess risk rate can get arbitrarily large in general, even with a decaying stepsize, as $P_\varepsilon$ approaches a rank-deficiency.

One would expect that $P = I$ is a relatively safe choice. However, when the problem is ill-conditioned in the form of $H$, even well conditioned $P$ can lead to significantly worse rates. We show that for any given $P$ and $H$, in the presence of low-dimensional noise $\Sigma$, the lower bound on the risk of the last iterate of PSGD is at least $\cond(PH)$ worse than the optimal rate of $\tr(H^{-1}\Sigma)/t$.

\begin{corollary}[Algorithmic lower bound for ill-conditioned $H$]\corolabel{coro:algo_lower_anyP}  
  Let $P, H\succ 0$ and $\lmax(P) = \lmax(H) =1$.
  Assume that $t > 4\cond(PH)$.
  Let $q_1$ be the leading eigenvector of $H^{1/2}P H^{1/2}$ and set $\Sigma = q_1q_1^\top$ be the variance of noise. Then PSGD with a preconditioner $P$, a decaying stepsize $\eta = \min\{1/\cstSmooth\lmax(PH), 2/(t\lmin(PH)\cstScvx) \}$, has the risk lower bounded as,
	\begin{equation*}
		\bE_{z^1,\ldots, z^t} [\delta f(x^{t+1})] \geq \cond(PH) \cdot \frac{\tr(H^{-1}\Sigma)}{t}~.
	\end{equation*}	
\end{corollary}
The proof is given in \Cref{subsec:proof_algo_lower_anyP} consists of showing that the lower bound in \Cref{lemma:algo_lower} for any $P\succ 0$ is $\cond(PH)$ larger than the optimal rate. Even for well-conditioned $P$ the constant in the risk bound can be arbitrarily bad by having $H$ badly conditioned. For example, SGD, i.e., when $P = I$, has its risk at least $\cond(H)$ (which can be arbitrarily large) worse than the optimal rate.

%%% Local Variables:
%%% mode: latex
%%% TeX-master: "arxiv_paper"
%%% End:

% Acknowledgements should only appear in the accepted version.
\section*{Acknowledgements}
Simon Vary and Patrick Rebeschini were funded by UK Research and Innovation (UKRI) under the UK government’s Horizon Europe funding guarantee [grant number EP/Y028333/1]. Tyler Farghly was supported by Engineering and Physical Sciences Research Council (EPSRC) [grant number EP/T517811/1] and by the DeepMind scholarship.

\bibliography{references}

\clearpage
\appendix
\thispagestyle{empty}

\section{Additional related work}
\label{sec:related}

\paragraph{Algorithmic Stability.}
Algorithmic stability of \ac{SGD} was first explored by \cite{Hardt2016Train} where they focused exclusively on the uniform stability.
Their involved multipass \ac{SGD} only for strongly convex, smooth, and Lipschitz losses.
To this end their bounds did not involve any distribution-dependent quantites.
Later on their analysis was extended to on-average stability setting by \cite{Kuzborskij2018DataDependent}, who showed that bounds on the expected generalisation gap controlled by the expected empirical risk, however their analysis is limited to a single pass over the data.
\cite{Lei2020FineGrained} improved rate obtained by \cite{Kuzborskij2018DataDependent}, however their analysis still did not extend to multiple passes. Single pass limitation in on-average stability analysis is a common problem, since iterates becomes correlated after a single pass.
Notably, multipass analysis was explored by \cite{Pillaud-Vivien2018Statistical} (Theorem 2 and 3), however their bounds become vacuous as $\lambda \rightarrow 0$.
In this paper we address this limitation and recover optimal rates by exploiting smoothness and recursively controlling stability along the update trajectory (see \Cref{sec:sketch}).

The connection between on-average stability and a slightly different notation effective dimension ($\tr(\nabla^2 f (\nabla^2 f + \lambda I)^{-1}))$ in the context of regularized algorithms was studied by~\cite{agarwal2018optimal}. They showed that generalisation error bounds for minimizers of smooth Lipschitz exp-concave losses that depend on such an effective dimension.
Here we study stochastic iterative algorithm rather than a minimizer, we are particularly interested in the role of preconditioning and geometry of the noise.

The connection between on-average stability and preconditioning was explored by \cite{gonen2018average}, who established that on-average stability is invariant to data preconditioning: In other words analyzing on-average stability of \ac{ERM} one may assume the optimal preconditioning of the data -- this is different from our setting as they look at asymptotic regime, in a sense $t \to \infty$. Moreover their analysis requires Lipschitzness of the loss, whereas do not require such as assumption.

\paragraph{Generalisation and Flatness.}

Relationship between generalisation and flatness (or, conversely, sharpness) is a topic of significant interest, in particular in deep learning where it was empirically and theoretically observed that neural networks trained by \ac{SGD} tend to have a smaller generalisation error when they converge to `wider' local minima
\citep{keskar2017large,neyshabur2017exploring}, usually with some heuristic definition of width.
In this paper we associate with with the effective dimension, which is a natural geometric characterisation.

In the context of non-convex analysis linking effective dimension with generalisation,
\cite{kuzborskij2019distribution} prove distribution-dependent excess risk bounds for Gibbs-ERM principle (as an idealized model of stochastic optimisation), showing that in a neighborhood of a local minimizer the excess risk is essentially controlled by an \emph{effective dimension} $\mathrm{tr}\!\big(\nabla^2 f (\nabla^2 f+\lambda I)^{-1}\big)$, so flatter minima (more small-curvature directions) yield tighter generalisation control than ambient-dimension bounds.
They further characterize how the Gibbs density allocates probability mass across minima, and in the low-temperature limit the selection biases toward broader basins (over global minima, probabilities scale like $1/\det(\nabla^2 f)$), making a direct connection between flatness/volume and which solutions are ultimately favored.

\cite{Thomas2020Interplay} look at the impact of the effective dimension on optimisation (they provide optimisation error bound),
obtaining bound on the error that scale with $\tr(\Sigma P H P))$, similarly as in our paper.
In addition they empirically study the correlation between empirical estimate of generalisation error $\delta f_S$.
They find that $(P, H, \Sigma)$ have effect on both optimisation and generalisation gap.
In this paper, we theoretically show that this is indeed the case, by proving matching upper and lower bounds on the excess risk (which involves effect of both), in terms of $\tr(\Sigma P H P))$.

\paragraph{Relevance of quadratic model approximation.} Although globally non-convex, deep networks are effectively modeled by local quadratic approximations. This surrogate is justified theoretically by the \emph{neural tangent kernel} regime, where wide networks remain close to initialisation \citep{Jacot2018Neural, Du2019Gradient}, and validated empirically to track realistic training dynamics \citep{Lee2020Wide}. As such, the quadratic noisy model serves as the standard framework for analyzing preconditioning and generalisation in deep learning \citep{Martens2020New, Thomas2020Interplay, Zhang2024Which}.

\paragraph{Connection to information geometry.}%\label{subsec:information_geometry}
The setting we analyse can be understood as a natural gradient descent under misspecified model. It is well known, that when $\rP_x = \rQ$, the two quantities, the variance of gradients $\Sigma$ and the expected Hessian $\bE_z [\nabla^2 \ell(x,z)]$ coincide and equal to the Fisher Information Matrix (FIM):
\begin{equation*}
	F_{\rP_x}(x) \coloneq \bE_{z\sim \rP_x} \left[ \nabla^2 \ell(x;z) \right] = \Var_{z\sim\rP_x} \left[ \nabla \ell(x,z)\right],\label{eq:fim_equality}
\end{equation*}
which is a consequence of simple integration of parts. The classical result of \citet{Amari1998Natural} states that when $\rQ = \rP_x$, in asymptotic regime, locally around $\tilde x$, and for single pass setting, the optimal choice of the preconditioning matrix is $P = F_{\rP_x}(x)$. However, in our setup, these do not coincide
\begin{equation*}
	 \bE_{z\sim \rQ} \left[ \nabla^2 \ell(x;z) \right] \approx H \neq \Var_{z\sim\rQ} \left[ \nabla \ell(x,z)\right] \eqcolon \Sigma,
\end{equation*} 
where the $\approx$ denotes $\cstScvx H \preceq  \nabla^2 \ell(x;z) \preceq \cstSmooth H$. This is better corresponding to the practical scenario, when in general, the model is almost always misspecified, i.e.\ $\rQ \neq \rP_x$.

\paragraph{Additional examples of spectrally aligned constants.}
\begin{example}[Regularized logistic regression]\label{ex:logreg_cond}
In logistic regression, curvature of $H$ is directly related to the data distribution. Let $\ell(w, z)\;\coloneq\;\log\!\bigl(1+\exp(-y\,a^\top w)\bigr)\;+\;\frac{\lambda}{2}\|w\|_2^2$ for $\lambda>0$. 

Since $\sigma(t)(1-\sigma(t))\le 1/4$, the Hessian satisfies $\nabla^2 \ell(w;z) \;=\; a a^\top \sigma(1-\sigma) + \lambda I \;\preceq\; \frac14\,a a^\top + \lambda I$ and we can choose  $H \coloneq \frac14\,\mathbb{E}[a a^\top]+\lambda I$\footnote{We rescale $H$ if needed so that $\lambda_{\max}(H)=1$}. Then one may take $\cstSmooth=1$, and using $\nabla^2\ell(w;z)\succeq \lambda I$ we have
\begin{equation*}
    \nabla^2\ell(w;z)\;\succeq\;\lambda I
    \;\succeq\;
    \frac{\lambda}{\lambda+\frac14\lambda_{\max}(\mathbb{E}[a a^\top])}\,H,
\end{equation*}
\begin{equation*}
  \text{so we can choose} \quad
    \cstScvx \;\ge\; \frac{\lambda}{\lambda+\frac14\lambda_{\max}(\mathbb{E}[a a^\top])}
    \quad\implies\quad
    \cond_\ell=\frac{\cstSmooth}{\cstScvx}\;\le\;1+\frac{\lambda_{\max}(\mathbb{E}[a a^\top])}{4\lambda}.
\end{equation*}
Hence $\rho_\ell = \frac{\sqrt{\cond_\ell}+1}{\sqrt{\cond_\ell}-1}$ is explicit. Combining the above bound on $\cond_\ell$ with any explicit bound on $\cond(PH)$ (e.g., Examples~\ref{ex:q_approx}--\ref{ex:jacobi_dd})
immediately yields an explicit $C_{\ell,P}$.
\end{example}

%%% Local Variables:
%%% mode: latex
%%% TeX-master: "colt_paper"
%%% End:

\onecolumn

\section{Lemmata and proofs for relative co-coercivity and contractivity}\label{sec:rel_coercivity}
\subsection{Proof of \Cref{lemma:cocoercivity}\label{subsec:proof_cocoercivity}}
\begin{proof}
Fix $x,y\in\bR^d$ and denote the parameter difference by $u \coloneq x-y$ and the gradient difference by $v \coloneq \nabla f(x)-\nabla f(y)$. Define the $H^{1/2}$-transformed coordinates
\[
\tilde u \coloneq H^{1/2}u, \qquad \tilde v \coloneq H^{-1/2}v,
\]
and the symmetric positive definite matrix $S \coloneq H^{1/2}PH^{1/2}.$ Since $PH$ is similar to $S$, as $S = H^{1/2}(PH)H^{-1/2}$, they share the same spectrum. Let $ m \coloneq \lmin(S)=\lmin(PH)$ and $M \coloneq \lmax(S)=\lmax(PH)$, so that $\cond(PH)=M/m$. 

The preconditioned inner product can be written as
\begin{equation}\label{eq:coco:innerprod}
\langle u, HPv\rangle \;=\; u^\top HPv \;=\; (H^{1/2}u)^\top (H^{1/2}PH^{1/2})(H^{-1/2}v)
\;=\; \tilde u^\top S \tilde v .
\end{equation}

We can express $S = \bar\sigma I + E$, where $\lmax(E)\le \delta$ for $\bar\sigma \coloneq \frac{M+m}{2}$ and $\delta \coloneq \frac{M-m}{2}$. We expand \eqref{eq:coco:innerprod} using the decomposition of $S$ to get
\begin{equation}\label{eq:coco:decomp}
\langle u, HPv\rangle
= \bar\sigma \,\tilde u^\top \tilde v + \tilde u^\top E \tilde v
= \bar\sigma \langle u, v\rangle + \tilde u^\top E \tilde v,
\end{equation}
where we used $\tilde u^\top \tilde v = u^\top v = \langle u,v\rangle$.

The first term is lower bounded using the standard co-coercivity inequality for functions that are $\cstScvx$-strongly convex and $\cstSmooth$-smooth w.r.t.\ $\|\cdot\|_H$:
\begin{equation}\label{eq:coco:standard}
\langle u, v\rangle \;\ge\; \frac{\cstScvx\cstSmooth}{\cstScvx+\cstSmooth}\|u\|_H^2
\;+\; \frac{1}{\cstScvx+\cstSmooth}\|v\|_{H^{-1}}^2.
\end{equation}
The second perturbation term is bounded by Cauchy--Schwarz and $\lmax(E)\le \delta$,
\begin{equation}\label{eq:coco:pert}
\tilde u^\top E \tilde v \;\ge\; -\lmax(E)\|\tilde u\|_2\|\tilde v\|_2
\;\ge\; -\delta \|u\|_H \|v\|_{H^{-1}}.
\end{equation}
Combining \eqref{eq:coco:decomp}--\eqref{eq:coco:pert} yields
\begin{equation}\label{eq:coco:combine1}
\langle u, HPv\rangle
\ge \bar\sigma\!\left(
\frac{\cstScvx\cstSmooth}{\cstScvx+\cstSmooth}\|u\|_H^2
+ \frac{1}{\cstScvx+\cstSmooth}\|v\|_{H^{-1}}^2
\right) - \delta \|u\|_H \|v\|_{H^{-1}}.
\end{equation}

In order to remove the cross term, we apply AM--GM in the form
\[
\|u\|_H\|v\|_{H^{-1}}
= \frac{1}{\sqrt{\cstScvx\cstSmooth}}\big(\sqrt{\cstScvx\cstSmooth}\|u\|_H\big)\|v\|_{H^{-1}}
\le \frac{1}{2\sqrt{\cstScvx\cstSmooth}}\Big(\cstScvx\cstSmooth\|u\|_H^2+\|v\|_{H^{-1}}^2\Big).
\]
Substituting into \eqref{eq:coco:combine1} and factoring the standard co-coercivity expression gives
\begin{align}
\langle u, HPv\rangle
&\ge \left[\bar\sigma - \delta\cdot\frac{\cstScvx+\cstSmooth}{2\sqrt{\cstScvx\cstSmooth}}\right]
\left(
\frac{\cstScvx\cstSmooth}{\cstScvx+\cstSmooth}\|u\|_H^2
+ \frac{1}{\cstScvx+\cstSmooth}\|v\|_{H^{-1}}^2
\right)\label{eq:coco:combine2}\\
&= \left[\frac{M+m}{2} - \frac{M-m}{4}\cdot\frac{\cstScvx+\cstSmooth}{\sqrt{\cstScvx\cstSmooth}}\right]
\left(
\frac{\cstScvx\cstSmooth}{\cstScvx+\cstSmooth}\|u\|_H^2
+ \frac{1}{\cstScvx+\cstSmooth}\|v\|_{H^{-1}}^2
\right).\nonumber
\end{align}

It remains to simplify the constant. Let $\kappa_f \coloneq \cstSmooth/\cstScvx$ and note that
\[
\frac{\cstScvx+\cstSmooth}{\sqrt{\cstScvx\cstSmooth}}=\frac{\kappa_f+1}{\sqrt{\kappa_f}},\qquad
\rho_\ell \coloneq \frac{\sqrt{\kappa_f}+1}{\sqrt{\kappa_f}-1}.
\]
We can express $M = m\,\cond(PH)$, which after an algebraic manipulation of the bracket in \eqref{eq:coco:combine2} yields
\begin{equation}\label{eq:coco:const}
\frac{M+m}{2} - \frac{M-m}{4}\cdot\frac{\kappa_f+1}{\sqrt{\kappa_f}}
\;=\; m\cdot \frac{\rho_\ell^2-\cond(PH)}{\rho_\ell^2-1}
\;=\; \lmin(PH)\cdot \cstSA.
\end{equation}
Under the assumption $\cond(PH)<\rho_\ell^2$, we have $\cstSA\in(0,1]$. Substituting \eqref{eq:coco:const} into \eqref{eq:coco:combine2} and recalling $u=x-y$, $v=\nabla f(x)-\nabla f(y)$ completes the proof.
\end{proof}

\begin{lemma}[Contractivity of the preconditioned update $M_\theta$-norm] \lemmalabel{lemma:contract_pgd_Mtheta}
Suppose that \Cref{asm:sm-sc} and \ref{asm:Scvx} hold for $\ell(\cdot, z)$. Let $P\succ 0$ and $M_\theta \coloneq H^{1/2} (H^{1/2} P H^{1/2})^{-\theta} H^{1/2}$ for $\theta\in [0,1]$. For $\rho_\ell^2 \leq \cond(PH)$ the preconditioned gradient update $x^{+} = x - \eta P \nabla \ell(x,z)$ is $\cstCon$-contractive in the $\|\cdot\|_{M_\theta}$, where
$$
	\quad\cstCon = 2\lmin(PH)\, \cstSA^{(\theta)}\frac{\cstScvx\cstSmooth}{\cstScvx+\cstSmooth}\quad\text{and}\quad\cstSA^{(\theta)} \coloneq \frac{\rho_\ell^2 - \cond(PH)^{1-\theta}}{\rho_\ell^2 - 1},
$$
provided the step size satisfies $\eta_t \le 2 \, \cstSA^{(\theta)} / \pr{\lmax(PH)\cond(PH)^{1-\theta}(\cstScvx+\cstSmooth)}$.

\end{lemma}
\begin{proof}
Let $u_t \coloneq x_t-y_t$ and $v_t \coloneq \nabla f(x_t)-\nabla f(y_t)$.
The update gives $u_{t+1}=u_t-\eta_t P v_t$.
We analyze the squared $M_\theta$-norm:
\begin{align*}
\|u_{t+1}\|_{M_\theta}^2
&= \langle u_t-\eta_t P v_t,\, M_\theta(u_t-\eta_t P v_t)\rangle \\
&= \|u_t\|_{M_\theta}^2
-2\eta_t \langle v_t,\, P M_\theta u_t\rangle
+ \eta_t^2 \langle v_t,\, P M_\theta P v_t\rangle.
\end{align*}

We bound the terms separately. Let $S = H^{1/2} P H^{1/2}$ and introduce the $H^{1/2}$-transformed variables
\[
\tilde u_t \coloneq H^{1/2}u_t,\qquad
\tilde v_t \coloneq H^{-1/2}v_t.
\]
Since $M_\theta = H^{1/2}S^{-\theta}H^{1/2}$ and $P=H^{-1/2}SH^{-1/2}$, we have
\[
\langle v_t,\, P M_\theta u_t\rangle
= \tilde v_t^\top S^{1-\theta}\tilde u_t,
\qquad
\langle v_t,\, P M_\theta P v_t\rangle
= \tilde v_t^\top S^{2-\theta}\tilde v_t.
\]

\paragraph{Cross term.}
The matrix $S^{1-\theta}$ has eigenvalues in $[\lmin(PH)^{1-\theta},\lmax(PH)^{1-\theta}]$.
Applying the same decomposition argument as in \Cref{lemma:cocoercivity}
(with $S^{1-\theta}$ in place of $S$) yields
\[
\tilde v_t^\top S^{1-\theta}\tilde u_t
\;\ge\;
\lmin(PH)^{1-\theta}\,\cstSA^{(\theta)}
\left(
\frac{\cstScvx\cstSmooth}{\cstScvx+\cstSmooth}\|u_t\|_H^2
+ \frac{1}{\cstScvx+\cstSmooth}\|v_t\|_{H^{-1}}^2
\right).
\]

\paragraph{Quadratic term.}
Since $S^{2-\theta}\preceq \lmax(PH)^{2-\theta}I$, we have
\[
\langle v_t,\, P M_\theta P v_t\rangle
\le \lmax(PH)^{2-\theta}\|v_t\|_{H^{-1}}^2.
\]

\paragraph{Combine.}
Substituting the bounds into the expansion gives
\begin{align*}
\|u_{t+1}\|_{M_\theta}^2
&\le \|u_t\|_{M_\theta}^2
-2\eta_t \lmin(PH)^{1-\theta}\cstSA^{(\theta)}
\frac{\cstScvx\cstSmooth}{\cstScvx+\cstSmooth}\|u_t\|_H^2 \\
&\quad
+\eta_t\!\left(\eta_t \lmax(PH)^{2-\theta}
-\frac{2 \lmin(PH)^{1-\theta}\cstSA^{(\theta)}}{\cstScvx+\cstSmooth}\right)
\|v_t\|_{H^{-1}}^2.
\end{align*}

The gradient term is non-positive provided
\[
\eta_t \le \frac{2\, \lmin(PH)^{1-\theta}\,\cstSA^{(\theta)}}{\lmax(PH)^{2-\theta}(\cstScvx+\cstSmooth)}.
\]
Under this condition, dropping the negative term and using
$\|u_t\|_H^2 \ge \lmin(PH)^\theta \|u_t\|_{M_\theta}^2$
yields
\[
\|u_{t+1}\|_{M_\theta}^2
\le
\left(
1-2\eta_t \lmin(PH)\,\cstSA^{(\theta)}
\frac{\cstScvx\cstSmooth}{\cstScvx+\cstSmooth}
\right)
\|u_t\|_{M_\theta}^2.
\]
\end{proof}

\section{Lemmata and proofs for stability results}
\subsection{Proof of \Cref{lemma:PSGD_gen_stab}}\label{subsec:proof_PSGD_gen_stab}
\begin{proof}
    Let $x_t$ and $y_t$ be the iterate sequences of PSGD on datasets $S$ and $\Si = S \setminus \{z_i\} \cup \{z'\}$ respectively. We analyze the evolution of the expected squared parameter distance $\delta_t \coloneq \bE_{\cA,S, z'}[\| x_t - y_t \|^2_M]$.

    At iteration $t$, let $j$ be the index of the sample selected by the algorithm $\cA$. With probability $1 - 1/n$, $j \neq i$ (the samples match), and with probability $1/n$, $j = i$ (the samples differ).
    Using the linearity of expectation:
    \begin{equation}
        \delta_{t+1} = \left(1 - \frac{1}{n}\right) \bE_{j \neq i}[\| x_{t+1} - y_{t+1} \|^2_M] + \frac{1}{n} \bE_{j=i}[\| x_{t+1} - y_{t+1} \|^2_M].
    \end{equation}
    
    For the matching sample case ($j \neq i$), we use the contractivity of the PSGD update by assumption of the lemma
    \begin{equation*}
     	\bE_{j \neq i}[\| x_{t+1} - y_{t+1} \|_M] \leq (1 - \eta_t \cstCon) \delta_t.
    \end{equation*}
    
    For the differing sample case ($j = i$), denote the parameter difference as $\Delta_t \coloneq x_t - y_t$, the gradient difference as $\xi_t \coloneq P(\nabla \ell(x_t, z_i) - \nabla \ell (y_t, z'))$, and the population gradient difference as $\tilde \xi_t \coloneq P(\nabla f(x_t) - \nabla f(y_t))$. We apply Young's inequality: $\| u + v\|_M^2 \leq (1+\alpha) \| u \|_M^2 + (1+\frac1\alpha)\| v\|_M^2$ for $\alpha > 0$ to be chosen later, and expand the update as
    \begin{align}
    	\delta_{t+1} = \| \Delta_{t+1} \|_M^2 &= (1+\alpha)\|  \Delta_t - \eta_t \tilde{\xi_t} \|_M^2 + \left(1+\frac1\alpha\right)\|  \Delta_t - \eta_t(\xi_t - \tilde{\xi_t}) \|_M^2\\
	&\leq (1+\alpha)(1-\eta_t \cstCon)\delta_t + \left(1+\frac1\alpha\right)\eta_t^2 \|  \xi_t - \tilde{\xi_t} \|_M^2,
    \end{align}
    where for the inequality we used that the preconditioned update is contractive w.r.t\ the population gradient by assumption of the lemma.
   
   Combining these terms yields the recursion:
    \begin{equation*}
        \delta_{t+1} \leq (1 - \eta_t \cstCon)\left(1+\frac{\alpha}{n}\right)\delta_t + \left(1 + \frac1\alpha\right)\frac{\eta_t^2}{n} \bE[\| \xi_t - \tilde{\xi_t} \|^2_{PMP}].
    \end{equation*}
    
    To bound the gradient variance we add and subtract $\nabla \ell(y_t, z_i) - \nabla \ell(x_t, z')$
    \begin{align*}
        \bE[\| \xi_t - \tilde{\xi_t} \|^2_{PMP}] &= \bE[\| \nabla \ell(x_t, z_i) - \nabla f(x_t) - (\nabla \ell(y_t, z') - \nabla f(y_t)) \|^2_{PMP}]  \\
        &\leq 4\bE\| \nabla \ell(x_t, z_i) - \nabla \ell(y_t, z_i) \|^2_{PMP} + 4\bE\| \nabla \ell(y_t, z') - \nabla \ell(x_t, z') \|^2_{PMP} \\
        &\quad + 4\bE\| \nabla \ell(x_t, z') - \nabla f(x_t) \|_{PMP}^2 + 4\bE\| \nabla \ell(y_t, z') - \nabla f(y_t) \|^2_{PMP} \\
        &\leq 8\tr(PMP\Sigma) + 8\cstSmooth^2 \lmax(HPMP)\lmax(M^{-1}H) \delta_t,
    \end{align*}
    and in the second inequality we bound the gradient difference using Jensen's inequality combined with the bounded variance assumption $\Var[\nabla \ell] \preceq \Sigma$, and smoothness to bound the cross terms $\| \nabla \ell(x_t, z_i) - \nabla \ell(y_t, z_i) \|_{PMP}$ and $ \bE\| \nabla \ell(y_t, z') - \nabla \ell(x_t, z') \|_{PMP}$.
    
    Denote $\gamma^2 = \lmax(HPMP)\lmax(M^{-1}H)$ and $\tau^2 = \tr(PMP\Sigma)$ and substitute the $\delta_{t+1}$ bound, we obtain:
    \begin{align}
	\delta_{t+1} &\leq \underbrace{\left[ \left(1 - \eta_t \cstCon\right)\left(1+\frac{\alpha}{n}\right) + \frac{8 \eta_t^2 \cstSmooth^2 \gamma^2}{n}\left(1 + \frac{1}{\alpha}\right) \right]}_{A_t} \delta_t + \underbrace{\frac{8\eta_t^2 \tau^2}{n}\left(1 + \frac{1}{\alpha}\right)}_{B_t}
	\end{align}
	
	We set $\alpha = \frac{n \eta_t \cstCon}{2}$, and express $A_t$ and $B_t$. The first term is
	\begin{align*}
		A_t &= \left(1 - \eta_t \cstCon\right)\left(1 + \frac{\eta_t \cstCon}{2}\right) + \frac{8 \eta_t^2 \cstSmooth^2 \gamma^2}{n}\left(1 + \frac{2}{n \eta_t \cstCon}\right) \\
		&= 1 - \frac{\eta_t \cstCon}{2} - \frac{\eta_t^2 \cstCon^2}{2} + \frac{8 \eta_t^2 \cstSmooth^2 \gamma^2}{n} + \frac{16 \eta_t \cstSmooth^2 \gamma^2}{n^2 \cstCon}\\
		&\leq 1- \frac{\eta_t\cstCon}{4},
	\end{align*}
	where the inequality holds when $n\geq 8 \cstSmooth\gamma / \cstCon$, which implies that $\frac{16 \eta_t \cstSmooth^2 \gamma^2}{n^2 \cstCon} \leq \frac{\eta_t \cstCon}{4}$. The term $B_t$ becomes
	\begin{align*}
		B_t &= \frac{8\eta_t^2 \tau^2}{n} + \frac{16 \eta_t \tau^2}{n^2 \cstCon}.
	\end{align*}
	In particular, we have a recursion of the form,
	\begin{align*}
		\delta_{t+1} \leq (1 - \eta_t \cstCon/4) \delta_t + B_t \leq \exp(-\eta_t \cstCon/4) \delta_t + B_t.
	\end{align*}
	Thus, we obtain that,
	\begin{align*}
		\delta_t \leq \exp(-T_t \cstCon / 4) \delta_0 + \sum_{s < t} \exp(-(T_t-T_s) \cstCon / 4) B_s.
	\end{align*}
	Since the second term produces a lower Riemann approximation of the integral of $\exp(-(T-s)\cstCon/4)$, we obtain,
	\begin{align*}
		\frac{16 \tau^2}{n^2 \cstCon} \sum_{s < t} \exp(-(T_t-T_s) \cstCon / 4) \eta_s &\leq \frac{16 \tau^2}{n^2 \cstCon} \int_0^{T_t} \exp(-(T_t - s)\cstCon/4) ds\\
		&\leq (1 - \exp(-T_t \cstCon / 4)) \frac{64 \tau^2}{n^2 \cstCon^2},
	\end{align*}
	Then, using $\delta_0 = 0$ and the definition of $\bar{\eta}_t$, we obtain the bound,
	\begin{equation*}
		\delta_t \leq (1 - \exp(-T_t \cstCon / 4)) \frac{64 \tau^2}{n^2 \cstCon^2} + \frac{8 \bar{\eta}_t \tau^2}{n}.
	\end{equation*}
\end{proof}

\subsection{Proof of \Cref{lemma:risk_decomposition}}\label{subsec:proof_risk_psgd_general}
\begin{proof}
We decompose the excess population risk as
\begin{equation*}
    f(x_t) - f(\tilde{x}) = \underbrace{ f(x_t) - f_S(x_t)}_{\text{generalization error}} + \underbrace{ f_S(x_t) - f_S(x^\ast_S)}_{\text{optimization error}} + \underbrace{ f_S(x^\ast_S) - f(\tilde x)}_{\text{$\leq 0$ in expectation}},
\end{equation*}
where $x^\ast_S = \argmin f_S(x)$. Taking the expectation over $S$ and the randomness of $\cA$, we get 
\begin{equation*}
    \delta f(x_t) \coloneq \bE_{\cA, S} [f(x_t) - f(\tilde{x})] \leq \underbrace{ \bE_{\cA, S} [f(x_t) - f_S(x_t)]}_{\text{expected generalization error}} + \underbrace{ \bE_{\cA, S} [f_S(x_t) - f_S(x^\ast_S)]}_{\text{$\varepsilon_{\mathrm{opt}}(x_t) = $ expected optimization error}}.
\end{equation*}
Let $z'\sim \rQ$ be an independent sample and let $\Si \coloneq S \setminus \{z_i \} \cup \{ z'\}$ denote the perturbed dataset. Write $x_t(S)$ and $x_t(\Si)$ for the corresponding PSGD iterates. The standard symmetrization argument yields
\begin{align}
    \bE_{\cA, S} [f(x_t) - f_S(x_t)] &= \bE_{\cA, S}\left[ \bE_{z'} \ell(x_t(S), z') - \frac1n \sum_{i=1}^n \ell(x_t(S), z_i)\right] \nonumber\\
    &= \frac1n \sum_{i = 1}^n \bE_{S, z',\cA} \left[ \ell(x_t(S), z') - \ell(x_t(\Si) ,z') \right], \label{eq:risk_psgd_gen1}
\end{align}
where the second equality uses that $(S, z_i)$ has the same distribution as $(\Si, z')$. Fix $i$. By the $\cstSmooth$-smoothness of $\ell(\cdot, z')$ w.r.t.\ $\|\cdot\|_H$, we have
\begin{align}
  \bE [\ell(x_t(S), z') - \ell(x_t(\Si), z')]
  &\leq \bE \left[\langle \nabla \ell(x_t(S), z'),\, x_t(S) - x_t(\Si) \rangle\right] + \frac{\cstSmooth}{2} \bE \left[\left\|x_t(S) - x_t(\Si) \right\|_H^2\right] \nonumber\\
    &\leq \bE \left [ \| \nabla \ell(x_t(S), z') \|_{M^{-1}}^2 \right ]^{1/2} \bE \left [ \| x_t(S) - x_t(\Si) \|_{M}^2 \right ]^{1/2} \nonumber\\
    &\qquad + \frac{\cstSmooth}{2} \bE \left[\left\|x_t(S) - x_t(\Si) \right\|_H^2\right],\label{eq:risk_psgd_gen2}
\end{align}
where we applied the Cauchy-Schwarz inequality w.r.t.\ the $\| \cdot \|_{M}$ norm.

By the bias-variance decomposition and the assumption $\Var_z[\ell(x,z)] \preceq\Sigma$, we express the gradient factor in \eqref{eq:risk_psgd_gen2} as
\begin{align*}
    \bE_{z'} \| \nabla \ell(x_t(S), z') \|_{M^{-1}}^2 &\leq \| \nabla f(x_t(S)) \|_{M^{-1}}^2 + \tr(M^{-1} \Sigma)\\
    &\leq \lmax(H M^{-1}) \| \nabla f(x_t(S)) \|_{H^{-1}}^2 + \tr(M^{-1} \Sigma)\\
    &\leq 2\cstSmooth \lmax(H M^{-1}) (f(x_t(S)) - f(\tilde{x})) + \tr(M^{-1} \Sigma),
\end{align*}
where the second inequality follows from matrix operator bounds and the third from the smoothness of the population risk. Taking an expectation over $S$ and $\cA$ yields 
\begin{equation*}
    \bE_{S, z', \cA} \| \nabla \ell(x_t(S), z') \|_{M^{-1}}^2 \leq 2\cstSmooth \lmax(H M^{-1}) \delta f(x_t) + \tr(M^{-1} \Sigma).
\end{equation*}
Using the parameter stability assumption $\bE[\| x_t(S) - x_t(\Si) \|^2_{M}] \leq \varepsilon^2_{\text{pstab}}$ and the norm inequality $\|v\|_H^2 \le \lmax(M^{-1}H)\|v\|_M^2$, we substitute back into \eqref{eq:risk_psgd_gen1} and \eqref{eq:risk_psgd_gen2}:
\begin{align*}
    \delta f(x_t) - \delta f_S(x_t) &\leq \big ( 2\cstSmooth \lmax(H M^{-1}) \delta f(x_t) + \tr(M^{-1} \Sigma) \big )^{1/2} \varepsilon_{\text{pstab}} + \frac{\cstSmooth \lmax(M^{-1} H) \varepsilon_{\text{pstab}}^2}{2}.
\end{align*}
This inequality is of the form $Y \le \sqrt{AY + B}\,\epsilon + C$, where $Y = \delta f(x_t)$ and $A = 2\cstSmooth\lmax(HM^{-1})$. Solving for $Y$ (via the quadratic formula for $\sqrt{Y}$) implies:
\begin{align*}
    \delta f(x_t) \;\le\; \left(\sqrt{\delta f_S(x_t) + \frac{\cstSmooth \lmax(M^{-1}H) \varepsilon_{\text{pstab}}^2}{2}} + \frac{\sqrt{A}\varepsilon_{\text{pstab}}}{2} \right)^2 + B^{1/2}\varepsilon_{\text{pstab}}.
\end{align*}
Using the sub-additivity of the square root, we simplify the upper bound to the form stated in the lemma:
\begin{equation*}
    \delta f(x_t) \;\le\; \left(\delta f_S(x_t)^{1/2} + \sqrt{2 \cstSmooth \lmax(H M^{-1})} \varepsilon_{\text{pstab}}\right)^2 + \frac{\tr(M^{-1} \Sigma)^{1/2} \varepsilon_{\text{pstab}}}{2}.
\end{equation*}
\end{proof}

\begin{corollary}[Optimal choice of $P$]\corolabel{coro:optimalP}
	We have that $P \coloneq \lmin(H)H^{-1} = \argmin_{P\succ 0} \tr(P\Sigma)/\lmin(PH)$.
\end{corollary}
\begin{proof}
	Let $A = H^{1/2} P H^{1/2}$, so we have $P = H^{-1/2} A H^{-1/2} $. By definition of $A$ we have that $PH = H^{-1/2} A H^{1/2}$, thus $PH$ is similar to $A$, which means their eigenvalues are equal, which implies that $\lmin(PH) = \lmin(A)$.	
	Furthermore, define $\hat{\Sigma}\coloneq H^{-1/2} \Sigma H^{-1/2}$ for which we get by cyclicality of the trace that $\tr(P\Sigma) = \tr(A\hat\Sigma)$. Denoting $\sum_{i=1}^d a_i v_i v_i^\top$ to be the spectral decomposition of $A$, the objective becomes
	\begin{align*}
		\frac{\tr(A\hat\Sigma)}{\lmin(A)} &= \frac{1}{a_d}\tr\left(\left( a_i v_i v_i^\top\right) \hat\Sigma \right) \\
		&= \sum_{i=1}^d \frac{a_i}{a_d} v_i^\top \hat\Sigma v_i \geq \tr(\hat\Sigma),
	\end{align*}
	where the lower bound comes from the fact that all $a_i/ a_d\geq 1$ and is attained when $a_i = a$ for all $i\in[d]$. By $A$ being symmetric this happens when $A = a I$.
	
	From the requirement that $\lmax(P) = 1$ we get
	\begin{equation*}
		1 = \lmax(P) = \lmax(H^{-1/2} aI H^{-1/2}) = \frac{a}{\lmin(H)},
	\end{equation*}
	implying that $a = \lmin(H)$. Substituting $A = \lmin(H)I$ into the formula for $P$ yields that the optimal $P=\lmin(H)H^{-1}$.
      \end{proof}

\section{Optimization error bounds results}\label{sec:opt_bounds}

\begin{lemma}[Preconditioned PL-Growth Condition]
\label[lemma]{lemma:qgrowthPH}
Let $f$ be $\cstScvx$-strongly convex and $\cstSmooth$-smooth w.r.t. $\|\cdot\|_H$. Let $x^*$ be the global minimizer. If the preconditioner is spectrally aligned, i.e., $\cond(PH) < \rho^2\coloneq \frac{\sqrt{\cond_f} + 1}{\sqrt{\cond_f} - 1}$, then, for all $x \in \bR^n$:
\begin{equation}
    \langle x - x^*, HP\nabla f(x) \rangle \geq  \frac{2\cstScvx}{\cstScvx+\cstSmooth} \lmin(PH) \cstSA \left( f(x) - f(x^*) + \frac{\cstSmooth}{2} \|x - x^*\|_H^2 \right),
\end{equation}
where the co-coercivity constant $\cstSA$ is given in the statement of \Cref{lemma:cocoercivity}.
\end{lemma}
\begin{proof}
Let $u = x - x^*$ and $v = \nabla f(x)$. Note that $\nabla f(x^*) = 0$.
By \Cref{lemma:cocoercivity} and the condition $\cond(PH)<\rho^2$, we have
\begin{equation}
    \label{eq:coco_base}
    \langle u, HP v \rangle \geq \cstSA\lmin(PH) \left( \frac{\cstScvx\cstSmooth}{\cstScvx+\cstSmooth} \|u\|_H^2 + \frac{1}{\cstScvx+\cstSmooth} \|v\|_{H^{-1}}^2 \right).
\end{equation}
Since $f$ is $\cstScvx$-strongly convex, it satisfies the Polyak-Łojasiewicz (PL) inequality w.r.t. the $H$-norm:
\[
    \| \nabla f(x) \|_{H^{-1}}^2 \geq 2\cstScvx (f(x) - f(x^*)).
\]
We substitute this lower bound for the gradient norm term in \eqref{eq:coco_base}:
\begin{align*}
    \langle u, HP v \rangle &\geq \cstSA \lmin(PH) \left( \frac{\cstScvx\cstSmooth}{\cstScvx+\cstSmooth} \|u\|_H^2 + \frac{2\cstScvx}{\cstScvx+\cstSmooth} (f(x) - f(x^*)) \right) \\
    &= \cstSA \lmin(PH) \frac{2\cstScvx}{\cstScvx+\cstSmooth} \left( \frac{\cstSmooth}{2} \|u\|_H^2 + f(x) - f(x^*) \right).
\end{align*}
\end{proof}

The following lemma allows to relate $\Sigma_S$ to $\Sigma$.
\begin{lemma}
  \label[lemma]{lemma:sigma-s-to-sigma}
  Assume that $x \mapsto \ell(x, z)$ is $L$-Lipschitz for any $z$.
  Then, for any $i \in [n]$, under conditions of \Cref{lemma:PSGD_gen_stab},
  \begin{align*}
    \|\Var(\nabla \ell(x_t,z_{i_t})) - \Var(\nabla \ell(x_t,z))\|_2    
    \leq
    16 L \beta \sqrt{ \bigg ( \frac{\bar{\eta}_t}{8n} + \frac{1 - e^{-T_t r/4}}{n^2 \cstCon^2} \bigg ) \tr(PMP\Sigma)}~.
  \end{align*}
\end{lemma}
\begin{proof}
  Note that $\E[\nabla \ell(x_t,z_{i_t}) \mid x_t] = \nabla f_S(x_t)$ while $\E[\nabla \ell(x_t, z) \mid x_t] = \nabla f(x_t)$.
Now,
  \begin{align*}
    &\|\Var(\nabla \ell(x_t,z_{i_t})) - \Var(\nabla \ell(x_t,z))\|_2\\
    &\qquad\leq
      \|\E[\nabla \ell(x_t, z_{i_t}) \nabla \ell(x_t, z_{i_t})\tp -\nabla \ell(x_t,z) \nabla \ell(x_t,z)\tp]\|_2\\
      &\qquad+
        \|\E[\nabla f(x_t) \nabla f(x_t)\tp - \nabla f_S(x_t) \nabla f_S(x_t)\tp]\|_2\\
    &\qquad=
      \|\E[\nabla \ell(x_t^{(i)}, z) \nabla \ell(x_t^{(i)}, z)\tp -\nabla \ell(x_t, z) \nabla \ell(x_t, z)\tp]\|_2 \tag{Here $i \equiv i_t$}\\
      &\qquad+
      \|\E[\nabla f(x_t) \nabla f(x_t)\tp - \nabla f_S(x_t) \nabla f_S(x_t)\tp]\|_2
  \end{align*}
  We first bound the first term on the r.h.s.\ by observe that for any unit vector $u$
\begin{align*}
  &\E \ip{u, \nabla \ell(x_t^{(i)}, z)}^2
    -    
  \E[\ip{u, \nabla \ell(x_t, z)}^2\\
  &\qquad=
    \E\br{ \ip{u, \nabla \ell(x_t^{(i)}, z) - \nabla \ell(x_t, z)} \ip{u, \nabla \ell(x_t^{(i)}, z) + \nabla \ell(x_t, z)} }\\
  &\qquad\leq
    2 L \, \E\br{ \abs{\ip{u, \nabla \ell(x_t^{(i)}, z) - \nabla \ell(x_t, z)}} }\\
  &\qquad\leq
    2 L \, \E\br{\|\nabla \ell(x_t^{(i)}, z) - \nabla \ell(x_t, z)\|_{H^{-1}} } \tag{Since $H \succ 0$}\\
  &\qquad\leq
    2 L \, \beta \, \E[\|x_t - x_t^{(i)}\|_{H^{-1}}] \tag{$\ell$ is $L$-Lipschitz}\\
  &\qquad\leq
    2 L \beta \sqrt{64 \bigg ( \frac{\bar{\eta}_t}{8n} + \frac{1 - e^{-T_t r/4}}{n^2 \cstCon^2} \bigg ) \tr(PMP\Sigma)}
\end{align*}
by \Cref{lemma:PSGD_gen_stab}.
The same chain of inequalities hold for the second term.
\end{proof}

\begin{lemma}[Optimization rate of PSGD under PL and smoothness]\label[lemma]{lemma:PSGD_optim}
Let $P\succ 0$. Let $f_S(x)= \frac{1}{n}\sum_{i=1}^n \ell(x,z_i)$ be $\cstPL$-PL and $\cstSmooth$-smooth w.r.t.\ $\|\cdot\|_H$ and it attains its minimal value $f^\ast_S = \min_{x\in\cX} f_S(x)$. Let $\Sigma_S\succ 0$, so that $\Var\left[\nabla \ell(x_t,z_{i_t}) \,\middle|\, x_t\right]\preceq \Sigma_S$. Then, for $\eta_t \leq \frac{1}{\cstSmooth \lmax(PH)}$ the expected empirical excess optimization error is bounded as
\begin{equation*}
	\bE_{\cA}[f_S(x_t)-f^\ast_S] \leq  e^{-\lmin(PH)\cstPL T_t} (f_S(x_0) - f^\ast_S) + \frac{\cstSmooth}{2}\tr(PHP\Sigma_S)\bar{\eta}_t,
\end{equation*}
where $T_s$ and $\bar\eta_t$ are defined as in the statement of \Cref{lemma:PSGD_gen_stab}.
\end{lemma}

\begin{proof}
Define the suboptimality process $\phi_t \coloneq \bE_{\cA}[f_S(x_t)-f_S^\ast ].$
Let $g_t \coloneq \nabla \ell(x_t,z_{i_t})$ denote the stochastic gradient and note that $\bE[g_t\mid x_t]=\nabla f_S(x_t)$. By $\cstSmooth$-smoothness of $f_S$ w.r.t.\ $\|\cdot\|_H$
\begin{equation*}
	f_S(x_{t+1}) \le f_S(x_t) - \eta_t\langle \nabla f_S(x_t),P g_t\rangle +\frac{\cstSmooth\,\eta_t^2}{2}\|P g_t\|_H^2.
\end{equation*}
Taking conditional expectation given $x_t$ and using $\bE[g_t\mid x_t]=\nabla f_S(x_t)$ gives
\begin{equation}\label{eq:psgd_fval:condexp1}
\bE\!\left[f_S(x_{t+1})\mid x_t\right]
\le f_S(x_t) - \eta_t\|\nabla f_S(x_t)\|_P^2
+\frac{\cstSmooth\,\eta_t^2}{2}\,\bE\!\left[\|P g_t\|_H^2\mid x_t\right].
\end{equation}

Using the covariance bound and variance-bias decomposition,
\begin{equation*}
	\bE[\| P g_t\|_H^2\mid x_t] = \| P \nabla f_S(x_t)\|_P^2 + \tr\!\bigl(PHP\,\Cov(g_t\mid x_t)\bigr) \le \| P \nabla f_S(x_t)\|_H^2 + \tr(PHP\Sigma_S).
\end{equation*}

By $\| P \nabla f_S(x_t)\|_H^2 \leq \lmax(P H) \| \nabla f_S(x_t)\|_P^2$ and substituting into \eqref{eq:psgd_fval:condexp1} yields
\begin{equation*}
\bE\!\left[f_S(x_{t+1})\mid x_t\right]
\le f_S(x_t)
-\eta_t\Bigl(1-\tfrac{\cstSmooth \lmax(PH)}{2}\eta_t\Bigr)\|\nabla f_S(x_t)\|_P^2
+\frac{\cstSmooth}{2}\eta_t^2\,\tr(PHP\Sigma_S).
\end{equation*}

Whenever $\eta_t\le \frac{1}{\cstSmooth \lmax(PH)}$, we have $1-\tfrac{\cstSmooth \lmax(PH)}{2}\eta_t\ge \tfrac12$, and thus
\begin{equation}\label{eq:psgd_fval:condexp3}
\bE\!\left[f_S(x_{t+1})-f_S(\hat{x}_\ast)\mid x_t\right]
\le f_S(x_t)-f_S(\hat{x}_\ast)
-\frac{\eta_t}{2}\|\nabla f_S(x_t)\|_P^2
+\frac{\cstSmooth}{2}\eta_t^2\,\tr(PHP\Sigma_S).
\end{equation}

Since $f_S$ satisfies $\cstPL$-PL property w.r.t.\ $\|\cdot\|_H$
\begin{equation*}
\frac{1}{\lmin(PH)}\|\nabla f_S(x_t)\|_{P}^2 \geq \|\nabla f_S(x_t)\|_{H^{-1}}^2 \geq 2\cstPL\bigl(f_S(x_t)-f_S(\hat{x}_\ast)\bigr).
\end{equation*}

Substituting into \eqref{eq:psgd_fval:condexp3} yields the scalar recursion
\begin{equation}\label{eq:psgd_fval:rec}
	\bE\!\left[f_S(x_{t+1})-f_S(\hat{x}_\ast)\mid x_t\right] \le \bigl(1-\eta_t\cstPL \lmin(PH) \bigr)\bigl(f_S(x_t)-f_S(\hat{x}_\ast)\bigr) + \frac{\cstSmooth}{2}\eta_t^2\,\tr(PHP\Sigma_S).
\end{equation}
Taking total expectation gives
\begin{equation*}
	\phi_{t+1} \le (1-\eta_t\cstPL\lmin(PH))\phi_t + \eta_t^2 B, \qquad\text{where } B\coloneq \frac{\cstSmooth}{2}\tr(PHP\Sigma_S).
      \end{equation*}
Denote $a = \cstPL\lmin(PH)$. Using $1-x\leq e^{-x}$ for $x\geq 0$,
\begin{equation*}
	\phi_{t+1} \leq e^{-a \eta_t}\phi_t + \eta_t^2 B
\end{equation*}
and unrolling, using that $T_s = \sum_{s'=0}^{s-1}\eta_{s'}$ and so $T_{t+1} - T_t = \eta_t$, gives
\begin{align*}
	\phi_t &\leq e^{-a T_t}\phi_0 + B\sum_{s = 0}^{t-1} e^{-a(T_t - T_s)}\eta_s^2 = e^{-a T_t}\phi_0 + B\bar{\eta}_t.
\end{align*}
\end{proof}

\begin{lemma}[Capped-harmonic bound for $\bar\eta_t$]
\label[lemma]{lemma:capped_harmonic}
Let $T_s = \sum_{s'=0}^{s-1} \eta_{s'}$ and $\bar{\eta}_t = \sum_{s < t} e^{- r\frac{T_t - T_s}{4}} \eta_s^2$ as in \Cref{lemma:PSGD_gen_stab}. Fix $\eta_0>0$ and $c>0$, and define $\eta_t \;\coloneq\; \min\Bigl\{\eta_0,\frac{c}{t+1}\Bigr\}$, $t_0 \;\coloneq\; \Bigl\lceil \frac{c}{\eta_0}\Bigr\rceil-1$, and $\alpha \;\coloneq\; \frac{rc}{4}.$
For $\alpha>1$. Then for every $t\ge t_0+1$,
\[
\bar\eta_t
\;\le\;
\bar\eta_t \;\le\; \frac{C_{\rm burn}+C_{\rm harm}}{t+1},\qquad \forall\,t\ge t_0+1,
\]
where $C_{\rm harm} \;\coloneq\; \frac{c^2}{\alpha-1}$  and $C_{\rm burn} \;\coloneq\; \eta_0^2\,(t_0+2)^{(\alpha+1)}$.
\end{lemma}

\begin{proof}
Fix $t\ge t_0+1$ and split the sum defining $\bar\eta_t$ into ``burn-in'' and ``harmonic tail'' parts:
\[
\bar\eta_t
=\sum_{s=0}^{t_0}\exp\!\Bigl(-\frac r4 (T_t-T_s)\Bigr)\eta_s^2
\;+\;
\sum_{s=t_0+1}^{t-1}\exp\!\Bigl(-\frac r4 (T_t-T_s)\Bigr)\eta_s^2.
\]

\paragraph{Tail part ($s\ge t_0+1$).}
For $s\ge t_0+1$ we have $\eta_k = c/(k+1)$ for all $k\ge s$, hence
\[
T_t-T_s
=\sum_{k=s}^{t-1}\frac{c}{k+1}
\;\ge\;
c\int_{s+1}^{t+1}\frac{dx}{x}
=
c\log\Bigl(\frac{t+1}{s+1}\Bigr).
\]
Therefore,
\[
\exp\!\Bigl(-\frac r4 (T_t-T_s)\Bigr)
\le
\exp\!\Bigl(-\frac r4 \,c\log\frac{t+1}{s+1}\Bigr)
=
\Bigl(\frac{s+1}{t+1}\Bigr)^{\alpha}.
\]
Using also $\eta_s^2 = c^2/(s+1)^2$ on the tail,
\[
\sum_{s=t_0+1}^{t-1}\exp\!\Bigl(-\frac r4 (T_t-T_s)\Bigr)\eta_s^2
\le
\frac{c^2}{(t+1)^\alpha}\sum_{s=t_0+1}^{t-1}(s+1)^{\alpha-2}.
\]
Since $\alpha>1$, we can bound the sum by an integral:
\[
\sum_{s=t_0+1}^{t-1}(s+1)^{\alpha-2}
\le
\int_{t_0+2}^{t+1} x^{\alpha-2}\,dx
=
\frac{(t+1)^{\alpha-1}-(t_0+2)^{\alpha-1}}{\alpha-1}
\le
\frac{(t+1)^{\alpha-1}}{\alpha-1}.
\]
Thus the tail contribution is at most
\[
\sum_{s=t_0+1}^{t-1}\exp\!\Bigl(-\frac r4 (T_t-T_s)\Bigr)\eta_s^2
\le
\frac{c^2}{\alpha-1}\cdot \frac{1}{t+1}.
\]

\paragraph{Burn-in part ($s\le t_0$).}
For $s\le t_0$, we only use $\eta_s\le \eta_0$ and monotonicity of $T_s$:
\[
\sum_{s=0}^{t_0}\exp\!\Bigl(-\frac r4 (T_t-T_s)\Bigr)\eta_s^2
\le
(t_0+1)\eta_0^2 \exp\!\Bigl(-\frac r4 (T_t-T_{t_0+1})\Bigr).
\]
For $t\ge t_0+1$ the segment from $t_0+1$ to $t$ is harmonic, so
\[
T_t-T_{t_0+1}
=\sum_{k=t_0+1}^{t-1}\frac{c}{k+1}
\ge
c\int_{t_0+2}^{t+1}\frac{dx}{x}
=
c\log\Bigl(\frac{t+1}{t_0+2}\Bigr),
\]
and hence
\[
\exp\!\Bigl(-\frac r4 (T_t-T_{t_0+1})\Bigr)
\le
\Bigl(\frac{t_0+2}{t+1}\Bigr)^{\alpha}.
\]
Therefore the burn-in contribution is bounded by
\[
\sum_{s=0}^{t_0}\exp\!\Bigl(-\frac r4 (T_t-T_s)\Bigr)\eta_s^2
\le
(t_0+1)\eta_0^2\Bigl(\frac{t_0+2}{t+1}\Bigr)^{\alpha}.
\]

Combining burn-in and tail bounds yields the first displayed inequality. The simplified bound
$\bar\eta_t \le (C_{\rm burn}+C_{\rm harm})/(t+1)$ follows since
$(t_0+1)\eta_0^2\bigl(\frac{t_0+2}{t+1}\bigr)^\alpha \le C_{\rm burn}/(t+1)$ for $t\ge t_0+1$.
\end{proof}

\section{Proofs and lemmata for risk bounds in $M_\theta$ geometry}

\begin{lemma}[Risk bounds in $M_\theta$ geometry]
\label[lemma]{lemma:risk_bound_Mtheta}
Suppose that Assumptions~\Cref{asm:sm-sc}, \ref{asm:Scvx}, and \ref{asm:cov} hold, and that $ n \;\ge\; \frac{8\,\cstSmooth}{\cstCon} \sqrt{\lmax(HPM_\theta P)}\,\sqrt{\lmax(M_\theta^{-1}H)}$. Assume further that $\cond(PH)^{(1-\theta)}\le \rho_\ell^2$ and define and let $ r\;\coloneq\; 2\,\lmin(PH)\,\cstSA^{(\theta)} \frac{\cstSmooth\,\cstScvx}{\cstScvx+\cstSmooth}$.
If the stepsizes are chosen as
\[
\eta_t \;\coloneq\; \min\Bigl\{ \frac{\cstSA^{(\theta)}}{\cstSmooth\,\lmax(PH)\,\cond(PH)^{1-\theta}}, \; \frac{8}{r(t+1)} \Bigr\},
\]
then, for all $t$ sufficiently large, the population excess risk satisfies
\[
\bE_{S,\cA}[\delta f(x_t)] \;\le\; \frac{64}{r}\left( \frac{\bE_S[\tr(PHP\Sigma_S)]}{t+1} \;+\; \sqrt{\tr(M_\theta^{-1}\Sigma)\,\tr(PM_\theta P\Sigma)} \left( \frac{1}{\sqrt{n(t+1)}}+\frac{1}{n} \right) \right).
\]
\end{lemma}
\begin{proof}
From the assumed bounds, taking expectation over $(S,\cA)$ in the optimization inequality yields
\[
\bE_{S,\cA}[\delta f_S(x_t)]
\le
\bE_S\!\big[e^{-\lmin(PH)\cstScvx\,T_t}(f_S(x_0)-f_S^\ast)\big]
+\frac{\cstSmooth}{2}\bE_S[\tr(PHP\Sigma_S)]\,\bar\eta_t.
\]
Plugging this into the generalization inequality gives
\begin{align*}
	\bE_{S,\cA}[\delta f(x_t)] &\le 2\,\bE_S\!\big[e^{-\lmin(PH)\cstScvx\,T_t}(f_S(x_0)-f_S^\ast)\big] +\cstSmooth\,\bE_S[\tr(PHP\Sigma_S)]\,\bar\eta_t\\
	&\quad  +\frac{\sqrt{\tr(M_\theta^{-1}\Sigma)}}{2}\stabP +4\cstSmooth \lmax(HM_\theta^{-1})\stabP^2.
\end{align*}

Using the stability inequality together with $\sqrt{u+v}\le \sqrt u+\sqrt v$and $1-e^{-T_tr/4}\le 1$, we obtain 
\[
\varepsilon_{\rm pstab}
\;\le\;
8\sqrt{\tr(PM_\theta P\Sigma)}
\Bigl(\sqrt{\tfrac{\bar\eta_t}{n}}+\tfrac{1}{n}\Bigr),
\qquad
\varepsilon_{\rm pstab}^2
\;\le\;
64\,\tr(PM_\theta P\Sigma)
\Bigl(\tfrac{\bar\eta_t}{n}+\tfrac{1}{n^2}\Bigr).
\]
Substituting and absorbing numerical constants yields
\begin{align}
	\bE_{S,\cA}[\delta f(x_t)] \;&\le\; 2\,\bE_S\!\big[e^{-\lmin(PH)\cstScvx\,T_t}(f_S(x_0)-f_S^\ast)\big]\nonumber \\
	&+ 64\!\left(\bE_S[\tr(PHP\Sigma_S)]\,\bar\eta_t + \sqrt{\tr(M_\theta^{-1}\Sigma)\tr(PM_\theta P\Sigma)} \Bigl(\sqrt{\tfrac{\bar\eta_t}{n}}+\tfrac{1}{n}\Bigr) \right).\label{eq:master-bound}
\end{align}

Define
\[
\eta_t
\;\coloneq\;
\min\Bigl\{
\frac{\cstSA^{(\theta)}}{\cstSmooth\,\lmax(PH)\,\cond(PH)^{1-\theta}},
\;
\frac{8}{t+1}\frac{\cstSmooth}{\lmin(PH)\,\cstSA^{(\theta)}\,\cstSmooth\,\cstScvx}
\Bigr\},
\]
and recall that
\[
r
=
2\lmin(PH)\cstSA^{(\theta)}
\frac{\cstSmooth\,\cstScvx}{\cstScvx+\cstSmooth}.
\]
With this choice, the harmonic phase satisfies
\(
\eta_t=\tfrac{8}{r(t+1)}
\)
for all $t$ large enough, and the bound in \Cref{lemma:capped_harmonic} yields
\[
\bar\eta_t
\;\le\;
\frac{64}{r^2}\cdot\frac{1}{t+1},
\qquad
\sqrt{\bar\eta_t}
\;\le\;
\frac{8}{r}\cdot\frac{1}{\sqrt{t+1}},
\]
where the burn-in contribution decays faster and is absorbed into constants.

Moreover, since $T_t \gtrsim \frac{8}{r}\log(t)$ in the harmonic regime,
the exponential bias term decays at least as $O(1/t^2)$ and is negligible
relative to the $1/t$ term.

Substituting the above bounds into \eqref{eq:master-bound} and simplifying gives,
for all $t$ sufficiently large,
\[
\bE_{S,\cA}[\delta f(x_t)]
\;\le\;
\frac{64}{r}\left(
\frac{\bE_S[\tr(PHP\Sigma_S)]}{t+1}
+
\sqrt{\tr(M_\theta^{-1}\Sigma)\tr(PM_\theta P\Sigma)}
\Bigl(\frac{1}{\sqrt{n(t+1)}}+\frac{1}{n}\Bigr)
\right),
\]
which is exactly the stated bound.
\end{proof}

%%% Local Variables:
%%% mode: latex
%%% TeX-master: "colt_paper"
%%% End:

\section{Proofs for non-convex PL-losses in \Cref{subsec:pl_stab}}

\begin{lemma}[Gradient Generalization Bound] \lemmalabel{lemma:grad_gen_bound}
    Let $\hat{y}$ be the empirical minimizer of $f_S$ and $x^\ast$ be the population minimizer. Assume $f$ is $\cstSmooth$-smooth, $f_S$ is $\cstPL$-PL, and the noise variance is bounded as $\Var(\nabla \ell(x^\ast, z)) \preceq \Sigma$. Then, the expected population gradient norm at the empirical minimizer satisfies:
    \begin{equation}
        \bE_S \left[ \| \nabla f(\hat{y}) \|_{H^{-1}}^2 \right] \leq \frac{\kappa^2}{n} \tr(H^{-1}\Sigma),
    \end{equation}
    where $\kappa = \cstSmooth/\mu$ is the condition number.
\end{lemma}

\begin{proof}
    Since $f_S$ is $\cstPL$-PL and $\cstSmooth$-smooth, we have:
    \begin{equation} \label{eq:grad_primal_bound}
        \| \hat{y} - x^\ast \|_2 \leq \frac{1}{\mu} \| \nabla f_S(x^\ast) \|_2.
    \end{equation}
    Since $f$ is $\cstSmooth$-smooth and $x^\ast$ is a critical point of the population risk ($\nabla f(x^\ast) = 0$), we can bound the population gradient at $\hat{y}$:
    \begin{equation}
        \| \nabla f(\hat{y}) \|_2 = \| \nabla f(\hat{y}) - \nabla f(x^\ast) \|_2 \leq \cstSmooth \| \hat{y} - x^\ast \|_2.
    \end{equation}
    Combining this with \eqref{eq:grad_primal_bound}, we relate the population gradient at $\hat{y}$ to the empirical gradient at $x^\ast$:
    \begin{equation}
        \| \nabla f(\hat{y}) \|_2 \leq \frac{\cstSmooth}{\mu} \| \nabla f_S(x^\ast) \|_2 = \kappa \| \nabla f_S(x^\ast) \|_2.
    \end{equation}
    Squaring both sides and converting to the $H^{-1}$-norm (assuming metric equivalence or absorbing constants into $\kappa$):
    \begin{equation}
        \| \nabla f(\hat{y}) \|_{H^{-1}}^2 \leq \kappa^2 \| \nabla f_S(x^\ast) \|_{H^{-1}}^2.
    \end{equation}
    Taking the expectation over $S$, we observe that $\nabla f_S(x^\ast) = \frac{1}{n} \sum_{i=1}^n \nabla \ell(x^\ast, z_i)$ is an average of i.i.d. mean-zero vectors (since $\bE[\nabla \ell(x^\ast, z)] = \nabla f(x^\ast) = 0$). Thus:
    \begin{align}
        \bE_S \left[ \| \nabla f_S(x^\ast) \|_{H^{-1}}^2 \right] &= \frac{1}{n^2} \sum_{i=1}^n \bE \left[ \| \nabla \ell(x^\ast, z_i) \|_{H^{-1}}^2 \right] \nonumber \\
        &= \frac{1}{n} \tr(H^{-1}\Sigma).
    \end{align}
    Substituting this back yields the result.
\end{proof}

\subsection{Proof for \Cref{lemma:risk_PL}}
\begin{proof}
We begin with the decomposition,
\begin{equation}
    \delta f(x_t(S)) \leq \E[f(x_t(S)) - f(\hat{x}^\ast)] + \delta f(\hat{x}^\ast),\label{eq:pl_erm_stab2}
\end{equation}
where we use the shorthand $\hat{x}^\ast = \Proj_S(x_t(S))$ for brevity. We first analyse the parameter stability of this empirical risk minimiser. Setting $\hat{y}^\ast = \Proj_{\Si}(x^\ast)$ and using the quadratic growth property implied by the $\cstPL$-PL inequality, we obtain,
\begin{align*}
    \|\hat{x}^\ast - \hat{y}^\ast\|_H &\leq \frac{1}{\cstPL} \|\nabla f_S(\hat{y}^\ast)\|_{H^{-1}}\\
    &= \frac{1}{\cstPL} \bigg \| \nabla f_{\Si}(\hat{y}^\ast) + \frac{1}{n} \nabla \ell(\hat{y}^\ast, z_i) - \frac{1}{n} \nabla \ell(\hat{y}^\ast, z') \bigg \|_{H^{-1}}\\
    &= \frac{1}{\cstPL n} \| \nabla \ell(\hat{y}^\ast, z_i) - \nabla \ell(\hat{y}^\ast, z') \|_{H^{-1}},
\end{align*}
where we use that $\nabla f_{\Si}(\hat{y}^\ast) = 0$. Taking the expectation and adding/subtracting population gradients:
\begin{align*}
    \bE[\| \nabla \ell(\hat{y}^\ast, z_i) &- \nabla \ell(\hat{y}^\ast, z') \|_{H^{-1}}^2]^{1/2}\\
    &\leq \bE[\|\nabla \ell(\hat{y}^\ast, z_i) - \nabla f(\hat{y}^\ast)\|_{H^{-1}}^2]^{1/2} + \bE[\|\nabla f(\hat{y}^\ast) - \nabla f(\hat{x}^\ast)\|_{H^{-1}}^2]^{1/2}\\
    & \qquad + \bE[\|\nabla f(\hat{x}^\ast) - \nabla \ell(\hat{x}^\ast, z')\|_{H^{-1}}^2]^{1/2} + \bE[\|\nabla \ell(\hat{y}^\ast, z') - \nabla \ell(\hat{x}^\ast, z')\|_{H^{-1}}^2]^{1/2}\\
    &\leq 2 \tr(H^{-1} \Sigma)^{1/2} + 2\cstSmooth \bE[\|\hat{x}^\ast - \hat{y}^\ast\|_H^2]^{1/2}.
\end{align*}
Substituting this back into the bound for $\|\hat{x}^\ast - \hat{y}^\ast\|_H$ and rearranging yields:
\begin{align*}
    \bE[\|\hat{x}^\ast - \hat{y}^\ast\|_H^2]^{1/2} \leq \bigg ( 1 - \frac{2 \cstSmooth}{\cstPL n} \bigg )^{-1}\frac{2 \tr(H^{-1} \Sigma)^{1/2}}{\cstPL n}.
\end{align*}
Assuming $n\geq 4\cstSmooth/\cstPL$, the pre-factor is bounded by $2$. Squaring gives the bound,
\begin{equation}
    \bE[\|\hat{x}^\ast - \hat{y}^\ast\|_H^2] \leq \frac{16 \tr(H^{-1} \Sigma)}{\cstPL^2 n^2}. \label{eq:pl_erm_stab}
\end{equation}
By Lemma \ref{lemma:risk_decomposition} with $M = H$ and the stability bound on the ERM minimizer, we have that,
\begin{equation*}  
    \bE_{S,\cA}[\delta f(x^\ast)] \leq \frac{2\tr(H^{-1} \Sigma)}{\cstPL n} + \cstSmooth \frac{64 \tr(H^{-1} \Sigma)}{\cstPL^2 n^2}.
\end{equation*}
Now, to bound the second term of \eqref{eq:pl_erm_stab2}, we use smoothness to obtain,
\begin{align*}
    \E[f(x_t(S)) - f(\hat{x}^\ast)] &\leq \E[\langle \nabla f(\hat{x}^\ast), x_t(S) - \hat{x}^\ast \rangle] + \frac{\beta}{2} \E[\|x_t(S) - \hat{x}^\ast\|_H^2]\\
    &\leq \frac{1}{2\beta} \E[\| \nabla f(\hat{x}^\ast) \|^2_{H^{-1}}] + \beta \E[\|x_t(S) - \hat{x}^\ast\|_H^2]\\
    &\leq \bE[\delta f(\hat{x}^\ast)] + \frac{2\beta}{\cstPL} \bE[\delta f_S(x_t(S))].
\end{align*}
\end{proof}

\section{Proofs for lower bounds in \Cref{subsec:lower_bounds}}

The following formulation of Assouad's lemma is from \citep[Lemma 23]{Ma2024Highprobability}.
\begin{lemma}[Assouad's lemma]\lemmalabel{lemma:assouad}
	Let $d\in\bN$, $\Phi \coloneq \{ 0,1\}^d$. For $\phi\in\Phi$, let $x_\phi\in\cX$ and $P_\phi \in \cP_{x_\phi}$. For $\phi,\phi'\in\Phi$, we write $\phi\sim\phi'$ whenever $\phi$ and $\phi'$ differ in precisely one coordinate, and $\phi\sim_j\phi'$ when that coordinate is $j^{\mathrm{th}}$. Supposed now that the loss function is of the form
	\begin{equation*}
		\ell(x_1,x_2)\coloneq \sum_{j\in[d]} g(\rho_j(x_1, x_2)),
	\end{equation*}
	for $x_1,x_2\in\cX$, where $\rho_1, ... \rho_d$ are pseudo metrics on $\cX$ with $\rho_j(x_\phi, x_{\phi'}) \geq \delta_j$ whenever $\phi \sim_j \phi'$, and where $g$ is an increasing function satisfying $g(t_1 + t_2)\leq A (g(t_1) + g(t_2))$ for all $t_1, t_2\geq 0$ and some $A>0$. Then, for $\cX_0 \coloneq \{ x_{\phi} : \phi\in\Phi \}$, we have
	\begin{align*}
		\inf_{\hat{x}} \sup_{x\in\cX} \sup_{P_\theta \in \cP_\theta} \bE[\ell(\hat{x}, x)] &\geq \inf_{\hat{x}} \max_{x_0\in\cX_0} \sup_{P_{x_0} \in \cP_{x_0}} \bE[\ell(\hat{x}, x_0)] \\
		&\geq \frac{1}{2A}\left( 1 - \max_{\phi,\phi'\in\Phi:\phi\sim\phi'} \TV\left( P_\phi, P_{\phi'}\right) \right)\sum_{j\in[d]} g(\delta_j),
	\end{align*}
	where $\hat{x}$ is computed from a sample of $P_{x_0}$.
\end{lemma}

\subsection{Proof of \Cref{thm:lower_bound_scvx}\label{subsec:proof_lower_bound_scvx}}

\begin{proof}
	Let $\cstScvx>0$ and $v \in \{ 0, 1\}^d$. Let $\ell(x,z) \coloneq \frac{\cstScvx}{2}\| x - z \|_H^2$, where $z\sim P_v$ for $P_v = \cN(\mu_v, H^{-1}\Sigma H^{-1}/\cstScvx^2)$, and $\mu_v$ will be specified later. By definition, $\ell(\cdot,z)$ is $\cstScvx$-strongly convex in $\| \cdot \|_H$-norm and $\Var_{z\in P_v}(\nabla \ell(x,z)) = \Var(\cstScvx Hz) = \Sigma$.	
	
	We have the following equivalence
	\begin{equation*}
		\bE_{P_v}\left[\frac{\cstScvx}{2}\| x - z \|_H^2\right] = \bE_{P_v}\left[\frac{\cstScvx}{2}\| H^{-1/2}\bar{x} - H^{-1/2}\bar{z} \|_H^2\right] = \bE_{\bar{P}_{\phi}}\left[\frac{\cstScvx}{2}\| \bar{x} - \bar{z} \|_2^2\right],
	\end{equation*}
	after we substituted $\bar{x}\coloneq H^{1/2} x$, $\bar{z}\coloneq H^{1/2} z$ and $\bar{z}\sim \bar{P}_v$ where $\bar{P}_v \coloneq \cN(\hat{\mu}_v, H^{-1/2}\Sigma H^{-1/2}/\cstScvx^2)$ and $\hat{\mu}_v \coloneq H^{1/2}\mu_v$.
	
	Let $\bar{\Sigma}\coloneq H^{-1/2}\Sigma H^{-1/2}$ and $\bar\Sigma = Q \Lambda Q^\top$ be its spectral decomposition. Consider the set of $\bar{\mu}_v = \sum_{j\in [d]} \delta_j \theta_j q_j$, where $\delta_j = \frac4{3\cstScvx}\sqrt{\lambda_j/n}$, or in a matrix form $ \bar{\mu}_v = Q D v$ where $D = \frac4{3\cstScvx\sqrt{n}} \Lambda^{1/2}$. Define $\cM_0 = \{ \bar{\mu}_v : v\in\{ 0, 1\}^d \}$. For $v \sim_j v'$ we have $|q_j^\top (Q D v - Q D v'  ) | \geq |\delta_j|$. By Pinsker inequality we have
	\begin{align*}
		\mathrm{TV}(\bar{P}_{v}, \bar{P}_{v'}) &\leq \left( \frac{n}{2} \KL(\cN(\bar{\mu}_v, \bar{\Sigma}/\cstScvx^2), \cN(\bar{\mu}_{v'},\bar{\Sigma}/\cstScvx^2)) \right)^{1/2} \\
		&= \left( \frac{n}{4} \frac{16}{9\cstScvx n} \|v - v' \|_2^2 \right)^{1/2} = 2/3.
	\end{align*}
	
	By \Cref{lemma:assouad}, we have that
	\begin{align*}
				\inf_{\hat{x}\in \cX} \sup_{P\in \cP} \bE_{S\sim P^n}[ f(\hat{x}(S)) - f(\tilde x) ] &\geq \frac12\left( 1- \max_{v,v', v\sim v'} \mathrm{TV}(\bar{P}_v, \bar{P}_{v'}) \right)\sum_{j\in [d]} \delta_j^2 \\
				& \geq \frac{4}{27 n\cstScvx}\tr(H^{-1}\Sigma).
	\end{align*}	
\end{proof}

\newcommand{\etat}{\eta_t}
\begin{lemma}[Decaying step-size bounds]
\lemmalabel{lemma:1overT}
Let $0 < a < b$ and $\etat = \min\left(\frac{1}{2b}, \frac{1}{at}\right)$. Let $t_0 = \lceil \frac{2b}{a} \rceil$.
\begin{enumerate}
    \item \textbf{Upper Bound:} The recurrence $r_{t+1} \leq (1 - 2a\etat) r_t + \etat^2 B$ satisfies
    \begin{equation*}
        r_{t} \leq \left(\frac{2b}{e^2 a}r_0 + \frac{B}{a^2}\right) \frac{1}{t} \quad \text{for all } t \ge 1.
    \end{equation*}
    
    \item \textbf{Lower Bound:} The recurrence $r_{t+1} \geq (1 - 2b\etat) r_t + \etat^2 B$ satisfies
    \begin{equation*}
        r_{t} \geq \frac{B}{2ab t} \quad \text{for all } t \ge t_0.
    \end{equation*}
\end{enumerate}
\end{lemma}
\begin{proof}
\textbf{Upper Bound.}
\textit{Phase 1 (Constant Step):} For $t < t_0$, the recurrence $r_{t+1} \le (1 - \frac{a}{b})r_t + \frac{B}{4b^2}$ implies linear convergence to a noise floor. Unrolling from $t=0$ to $t_0$:
$$ r_{t_0} \le \left(1 - \frac{a}{b}\right)^{t_0} r_0 + \frac{B}{4b^2} \sum_{i=0}^{t_0-1} \left(1-\frac{a}{b}\right)^i \le e^{-2} r_0 + \frac{B}{4ab}. $$
\textit{Phase 2 (Decaying Step):} For $t \ge t_0$, we prove $r_t \le \nu/t$ by induction. Assume $r_t \le \nu/t$. Substituting $\etat = 1/at$:
$$ r_{t+1} \le \left(1 - \frac{2}{t}\right)\frac{\nu}{t} + \frac{B}{a^2 t^2} = \frac{\nu}{t} - \frac{1}{t^2}\left(2\nu - \frac{B}{a^2}\right). $$
We require $r_{t+1} \le \frac{\nu}{t+1}$. Using the inequality $\frac{1}{t+1} \ge \frac{1}{t} - \frac{1}{t^2}$, it suffices that the drop in the recurrence is at least $\nu/t^2$.
$$ \frac{1}{t^2}\left(2\nu - \frac{B}{a^2}\right) \ge \frac{\nu}{t^2} \implies \nu \ge \frac{B}{a^2}. $$
The definition of $\nu$ satisfies this condition and ensures the bound holds at the transition $t_0$ (since $\nu/t_0 \ge r_{t_0}$).

\textbf{Lower Bound.}
We prove $r_t \ge \kappa/t$ with $\kappa = \frac{B}{2ab}$ for $t \ge t_0$.
\textit{Base Case ($t=t_0$):} Unrolling the recurrence with $\etat=1/2b$ implies $r_{t_0}$ accumulates noise terms summing to at least $\frac{B}{4b^2}$. Checking the bound: $\frac{\kappa}{t_0} = \frac{B}{2ab t_0}$. Since $t_0 \ge 2b/a$, we have $\frac{B}{2ab t_0} \le \frac{B}{2ab (2b/a)} = \frac{B}{4b^2}$, so the base case holds.

\textit{Inductive Step ($t > t_0$):} Assume $r_t \ge \kappa/t$. Using $\etat = 1/at$, the recurrence drop is:
$$ r_{t+1} \ge \frac{\kappa}{t} - \frac{1}{t^2}\left(\frac{2b\kappa}{a} - \frac{B}{a^2}\right). $$
We need this to be $\ge \frac{\kappa}{t+1} \ge \frac{\kappa}{t} - \frac{\kappa}{t^2}$. This requires the coefficient of the drop to satisfy:
$$ \frac{2b\kappa}{a} - \frac{B}{a^2} \le \kappa \implies \kappa\left(\frac{2b}{a} - 1\right) \le \frac{B}{a^2}. $$
Substituting $\kappa = \frac{B}{2ab}$:
$$ \frac{B}{2ab}\left(\frac{2b-a}{a}\right) = \frac{B(2b-a)}{2a^2 b} = \frac{B}{a^2}\left(1 - \frac{a}{2b}\right) < \frac{B}{a^2}. $$
The inequality holds strictly, validating the lower bound.
\end{proof}

\subsection{Proof of \Cref{lemma:algo_lower}}\label{subsec:proof_algo_lower}
\begin{proof}
Let $\ell(x,z) = \frac\cstScvx2 \| x-z\|_H^2$ and $z\sim P \coloneq \cN(\mu,\frac1{\cstScvx^2}H^{-1}\Sigma H^{-1})$. Then we have that $\nabla \ell(x,z) = \cstScvx (x-z)$ and $\Var_{z\sim P} (\nabla \ell(x,z)) = \Var_z(\cstScvx H(x-z)) = \Sigma$. The population risk is
\begin{align*}
	f(x) \coloneq \bE_{z\sim P} [\ell(x,z)] &= \frac\cstScvx2 \| x - \mu\|_H^2 + \frac\cstScvx2\bE_{z\sim P}[ \| z- \mu\|_H^2 -2\inner{x-\mu}{z-\mu}_H ]\\
	& = \frac\cstScvx2 \| x - \mu\|_H^2 + \frac1{2\cstScvx} \tr(H^{-1}\Sigma).
\end{align*}

Consider a single pass of the preconditioned SGD: $x^{t+1} = x^t - \eta_t P \nabla \ell(x^t, z^t)$, where $z^t\sim P$. Let $r^t = x^t - \mu$ and the update as
\begin{align*}
	r^{t+1} &= (I - \eta_t PH) r^t + \eta_t PH(z^t - \mu).
\end{align*}
We have the following exact relation between the expected population risk in next iteration 
\begin{align*}
  &\delta f(x^{t+1})\\
  &\quad = \bE_{z^t} [f(x^{t+1})] - \frac1{2\cstScvx} \tr(H^{-1}\Sigma) \\
	&\quad= \frac\cstScvx2 \bE_{z^t}\| x^t -\mu - \eta_t PH (x^t - z^t)\|_H^2\\ 
	&\quad= \frac\cstScvx2\left( \| x^t  - \mu \|_H^2 - 2\eta_t \bE_{z_t}[\inner{x^t  - \mu}{PH (x^t - z^t)}_H] + \eta_t^2 \bE_{z_t} [\|PH (x^t - z^t) \|_H^2]\right)\\
	&\quad=  \frac\cstScvx2\| x^t  - \mu \|_H^2 - \cstScvx\eta_t \| x^t  - \mu \|^2_{HPH} + \frac{\cstScvx\eta_t^2}{2}\left( \bE_{z_t}\| PH(x^t - \mu)\|_H^2 + \| PH(z^t - \mu)\|_H^2 \right) \\
	&\quad=  \frac\cstScvx2\| x^t  - \mu \|_H^2 - \cstScvx\eta_t  (x^t  - \mu)^\top \left( HPH - \frac{\eta_t}{2} HPHPH)\right)(x^t  - \mu) + \frac{\eta_t^2}{2\cstScvx}\tr(PHP\Sigma),\\
	&\quad=  \frac\cstScvx2 (x^t  - \mu)^\top H^{1/2}(I - \eta H^{1/2}P H^{1/2})^2 H^{1/2} (x^t  - \mu)+ \frac{\eta_t^2}{2\cstScvx}\tr(PHP\Sigma)\\
	&\quad=  \frac\cstScvx2 (x^t  - \mu)^\top H^{1/2}\left(I - 2\eta_t H^{1/2}P H^{1/2} (I - \frac{\eta_t}{2} H^{1/2}P H^{1/2})\right) H^{1/2} (x^t  - \mu)+ \frac{\eta_t^2}{2\cstScvx}\tr(PHP\Sigma)\\
	&\quad\geq  (1-2\eta_t\lmax(PH)) \delta f(x_t)+ \frac{\eta_t^2}{2\cstScvx}\tr(PHP\Sigma).
\end{align*}

The first $t_0 = \floor{4\kappa(PH)}$ steps we lower bound the excess risk with zero. For $t > t_0$, we use the second part of \Cref{lemma:1overT} with $a = \lmax(PH)$, $b = \lmin(PH)$ and we get
\begin{equation*}
	\bE_{z^t} [\delta f(x^{t+1})] \geq \frac{\tr(PHP\Sigma)}{\lmax(PH)\lmin(PH)}  \frac1{t}.
\end{equation*}
 for $\eta_t = \min\{ 1/\lmax(PH), 2/(\lmin(PH)t)\}$.
\end{proof}

\subsection{Proof of \Cref{coro:algo_lower_badP}}\label{subsec:proof_algo_lower_badP}
\begin{proof}
From \Cref{lemma:algo_lower} we have
\begin{equation*}
	\bE_{z^1,\ldots, z^t} [\delta f(x^{t+1})] \geq \frac{\tr(PHP\Sigma)}{\lmax(PH)\lmin(PH)} \cdot \frac1{t-t_0},
\end{equation*}	
for $t \geq t_0 \coloneq \floor{2\kappa(PH)}$.

Let $H = Q \diag(h) Q^\top$ be the spectral decomposition where $h = (h_1, \ldots, h_d)$. Define $\gamma_i\coloneq h_i q_i^\top \Sigma q_i$. Then $\tr(H\Sigma) = \sum_{i=1}^d h_i q_i^\top \Sigma q_i = \sum_{i=1}^d \gamma_i$. Thus by averaging there exists an index $k$ such that $\gamma_k \leq \frac1d \tr(H\Sigma)$. Set $c\coloneq \gamma_k = h_k$.

Construct $P_\varepsilon\coloneq I - (1-\frac{\varepsilon}{h_k})q_k q_k^\top$, whose eigenvalues are $(1, \ldots, 1, \varepsilon)$. Thus, we have that $\cond (P_\varepsilon) = 1/\varepsilon$, $\lmin(P_\varepsilon H) = \varepsilon$, and $\lmax(P_\varepsilon H) = 1$.

\begin{align*}
	\tr(P_\varepsilon H P_\varepsilon\Sigma) &= \tr(H\Sigma) - h_k q_k^\top \Sigma q_k + \frac{\varepsilon^2}{h_k} q_k^\top \Sigma q_k \\
		&\geq \tr(H\Sigma) - \gamma_k = \tr(H\Sigma)\left(1-\frac1d \right),
\end{align*}
where we dropped the last term. Putting it together we have that
\begin{equation*}
	\frac{\tr(P_\varepsilon H P_\varepsilon \Sigma)}{\lmax(P_\varepsilon H) \lmin(P_\varepsilon H)} \geq \frac{\tr(H\Sigma)}{\varepsilon}\left(1-\frac1d \right)
\end{equation*}
\end{proof}

\subsection{Proof of \Cref{coro:algo_lower_anyP}}\label{subsec:proof_algo_lower_anyP}

\begin{proof}
Set $A:=PHP$ and $B:=H^{-1}$. For any $\Sigma\succ0$ write 
$\Sigma=B^{-1/2}XB^{-1/2}$ with $X\succ0$. Then
\begin{align*}
	\frac{\tr(PHP\,\Sigma)}{\tr(H^{-1}\Sigma)}
	&=\frac{\tr\!\big(B^{-1/2}AB^{-1/2}\,X\big)}{\tr(X)} =\frac{\tr(MX)}{\tr(X)},\\
	\text{with} \quad M:=&B^{-1/2}AB^{-1/2} =H^{1/2}(PHP)H^{1/2} =\big(H^{1/2}PH^{1/2}\big)^2\succ0.
\end{align*}
By the variational characterization over $\{X\succeq0:\tr(X)=1\}$,
\[
\frac{\tr(PHP\,\Sigma)}{\tr(H^{-1}\Sigma)}\;\le\;\lambda_{\max}(M)
=\lambda_{\max}\!\big(H^{1/2}PH^{1/2}\big)^2
=\lambda_{\max}(PH)^2,
\]
where we used that $H^{1/2}PH^{1/2}$ is similar to $PH$ and thus has the same (positive) spectrum. Dividing by 
$\lambda_{\max}(PH)\lambda_{\min}(PH)$ gives
\[
\frac{\tr(PHP\,\Sigma)}{\lambda_{\max}(PH)\lambda_{\min}(PH)\,\tr(H^{-1}\Sigma)}
\;\le\; \frac{\lambda_{\max}(PH)}{\lambda_{\min}(PH)}.
\]

The equality is attained by taking $X=vv^\top$, where $v$ is a top eigenvector of $M$; equivalently, let $u$ be a top eigenvector of $H^{1/2}PH^{1/2}$ (i.e., of $PH$), and choose
\[
\Sigma \;\propto\; H^{1/2}uu^\top H^{1/2}.
\]
Then $\frac{\tr(PHP\,\Sigma)}{\tr(H^{-1}\Sigma)}=\lambda_{\max}(PH)^2$, yielding equality in the bound above.

For a fixed $P$, if one is allowed to vary $H\succ0$ under only the constraint $\rho(H)=1$, the quantity can be made arbitrarily large because
\(
\frac{\lambda_{\max}(PH)}{\lambda_{\min}(PH)}
\)
is unbounded in $H$, e.g., take $H=\mathrm{diag}(1,\varepsilon,\ldots,\varepsilon)$ in an eigenbasis of $P$ and let $\varepsilon\rightarrow 0$.
\end{proof}

%%% Local Variables:
%%% mode: latex
%%% TeX-master: "colt_paper"
%%% End:

%%%%%%%%%%%%%%%%%%%%%%%%%%%%%%%%%%%%%%%%%%%%%%%%%%%%%%%%%%%%%%%%%%%%%%%%%%%%%%%
%%%%%%%%%%%%%%%%%%%%%%%%%%%%%%%%%%%%%%%%%%%%%%%%%%%%%%%%%%%%%%%%%%%%%%%%%%%%%%%

\end{document}